\documentclass[letterpaper]{article} % DO NOT CHANGE THIS
\usepackage{aaai2026}  % DO NOT CHANGE THIS
\usepackage{times}  % DO NOT CHANGE THIS
\usepackage{helvet}  % DO NOT CHANGE THIS
\usepackage{courier}  % DO NOT CHANGE THIS
\usepackage[hyphens]{url}  % DO NOT CHANGE THIS
\usepackage{graphicx} % DO NOT CHANGE THIS
\usepackage{natbib}  % DO NOT CHANGE THIS AND DO NOT ADD ANY OPTIONS TO IT
\usepackage{caption} % DO NOT CHANGE THIS AND DO NOT ADD ANY OPTIONS TO IT
\usepackage{algorithmic}
\usepackage{bm}
\usepackage{amssymb}
\usepackage{amsmath}
\usepackage{amsthm}
\usepackage{multirow}
\newtheorem{definition}{Definition}
\newtheorem{lemma}{Lemma}
\usepackage[ruled,linesnumbered]{algorithm2e}
\usepackage{subfigure}
\usepackage{appendix}

\title{Prototype-Based Semantic Consistency Alignment for Domain Adaptive Retrieval}
\author{
   Tianle Hu\textsuperscript{\rm 1}\thanks{Code is available at: \protect\url{https://github.com/SkyHappyHu}},
    Weijun Lv\textsuperscript{\rm 2},
    Na Han\textsuperscript{\rm 3},
    Xiaozhao Fang\textsuperscript{\rm 2}\thanks{Corresponding author},
    Jie Wen\textsuperscript{\rm 4}, Jiaxing Li\textsuperscript{\rm 5}, Guoxu Zhou\textsuperscript{\rm 2}
}
\affiliations{
    \textsuperscript{\rm 1}School of Computer Science and Technology, Guangdong University of Technology, Guangzhou 510006, China\\
    \textsuperscript{\rm 2}School of Automation, Guangdong University of Technology, Guangzhou 510006, China\\
    \textsuperscript{\rm 3}School of Computer Science, Guangdong Polytechnic Normal University, Guangzhou 510665, China\\
    \textsuperscript{\rm 4}School of Computer Science and Technology, Harbin Institute of Technology, Shenzhen 518055, China\\
    \textsuperscript{\rm 5}School of Artificial Intelligence, Guangzhou University, Guangzhou 510006, China\\
    \{hutianlegdut, lvweijun0201\}@163.com, \{hannagdut, xzhfang168, jiewen\_pr\}@126.com, jiaxing.li.cs@gmail.com, gx.zhou@gdut.edu.cn
}

\usepackage{bibentry}
\begin{document}

\maketitle

\begin{abstract}
Domain adaptive retrieval aims to transfer knowledge from a labeled source domain to an unlabeled target domain, enabling effective retrieval while mitigating domain discrepancies. However, existing methods encounter several fundamental limitations: 1) neglecting class-level semantic alignment and excessively pursuing pair-wise sample alignment; 2) lacking either pseudo-label reliability consideration or geometric guidance for assessing label correctness; 3) directly quantizing original features affected by domain shift, undermining the quality of learned hash codes. In view of these limitations, we propose Prototype-Based Semantic Consistency Alignment (PSCA), a two-stage framework for effective domain adaptive retrieval. In the first stage, a set of orthogonal prototypes directly establishes class-level semantic connections, maximizing inter-class separability while gathering intra-class samples. During the prototype learning, geometric proximity provides a reliability indicator for semantic consistency alignment through adaptive weighting of pseudo-label confidences. The resulting membership matrix and prototypes facilitate feature reconstruction, ensuring quantization on reconstructed rather than original features, thereby improving subsequent hash coding quality and seamlessly connecting both stages. In the second stage, domain-specific quantization functions process the reconstructed features under mutual approximation constraints, generating unified binary hash codes across domains. Extensive experiments validate PSCA's superior performance across multiple datasets.
\end{abstract}


\section{Introduction}
Hashing receives extensive attention in the field of image retrieval due to its merits of compact storage and computational efficiency. The main purpose of hashing is to develop effective hash functions that preserve similarity relationships of original data in binary Hamming space. Several methods, such as Spectral hashing (SH) \cite{weiss2008spectral}, Density Sensitive Hashing (DSH) \cite{liu2016deep} and Scalable Graph Hashing (SGH) \cite{jiang2015scalable} endeavor to preserve pair-wise similarity of original data within the Hamming space. Ordinal Constraint Hashing (OCH) \cite{liu2018ordinal} introduces the ordinal relation in learning to hash. Iterative Quantization (ITQ) \cite{gong2012iterative} focuses on maintaining the locality structure by improving the consistency between generated discrete codes and their corresponding continuous representations.

Nonetheless, these aforementioned methods assume that queries and retrieved images share consistent data distributions, limiting their applicability in complex real-world scenarios. For instance, online shopping platforms showcase product images shot under ideal conditions, whereas user-submitted query photos typically contain cluttered backgrounds. To bridge this non-negligible domain gap \cite{hu2025coarse}, Domain Adaptation (DA) \cite{zhang2023hybrid,zhang2023fine} is combined with hashing, giving rise to a promising research field, i.e., Domain Adaptive Retrieval (DAR).

DAR encompasses two retrieval scenarios, i.e., single-domain retrieval and cross-domain retrieval. The former supposes both queries and retrieved samples originate from the target domain. In the context of cross-domain retrieval, the source domain serves as the retrieved dataset while queries stem from the target domain. Recently, several DAR methods are proposed. Probability Weighted Compact Feature (PWCF) \cite{huang2020probability} utilizes a focal-triplet constraint to mitigate the domain gap in a domain-invariant subspace. Domain Adaptation Preconceived Hashing (DAPH) \cite{huang2021domain} introduces Maximum Mean Discrepancy (MMD) \cite{gretton2012kernel} to prompt the domain marginal distribution alignment. These geometry-oriented methods lack consideration of semantic relationships between features and labels, resulting in suboptimal performance when significant semantic misalignment exists. Consequently, subsequent methods shift their focus toward incorporating semantic guidance. Two-Step Strategy (TSS) \cite{chen2023two} proposes a discriminative semantic fusion for hash learning. Semantic Guided Hashing Learning (SGHL) \cite{zhang2023semantic} and Dynamic Confidence Sampling and Label Semantic Guidance (DCS-LSG) \cite{zhang2024dynamic} further align the cross-domain conditional distributions by integrating category labels.

Despite their promising performance, we identify certain critical limitations of current DAR methods: 1) excessive focus on pair-wise sample alignment. Specifically, PWCF, TSS, SGHL and DCS-LSG primarily minimize distribution discrepancies between semantically consistent sample pairs, which suffer from computational inefficiency and limited distributional coverage of data \cite{yuan2020central}. 2) inadequate handling of pseudo-label reliability. Pseudo-labeling serves to predict the latent semantic associations between classes and unlabeled data, consequently providing fully annotated data to facilitate knowledge transfer. However, existing methods typically adopt off-the-shelf strategies, neglecting correction mechanisms for erroneous predictions. This inevitably leads to biased domain alignment and degraded hash codes quality. Although the most recent method, DCS-LSG, considers pseudo-label noise, it relies solely on semantic consensus between dual projections, without incorporating geometric knowledge for reliability assessment. 3) directly mapping original features with imperfect domain alignment to Hamming space, resulting in high quantization errors and limited discriminative power of generated codes.

To systematically tackle the limitations mentioned above, we propose the Prototype-Based Semantic Consistency Alignment (PSCA) framework. The core innovation lies in a semantic consistency alignment that evaluates pseudo-label reliability by comparing geometric proximity with semantic predictions, adaptively weighting the pseudo-labels. To be precise, in the first stage, PSCA establishes orthogonal class prototypes within a domain-shared subspace, where the semantic consistency alignment performs as follows: when geometric and semantic indicators agree, semantic weights are adjusted based on decision margins, as larger margins reflect stronger prediction confidence; when they conflict, semantic contribution is reduced proportionally to the disagreement level. This process derives a soft membership matrix that guides prototype learning in turn, thereby capturing more reliable semantic connections that mitigate error propagation.

After stage one, the membership matrix and prototypes reconstruct enhanced features, ensuring superior coding quality by circumventing direct quantization strategies. In the second stage, domain-specific quantization functions quantize the reconstructed features under mutual approximation constraints, capturing domain-specific characteristics while facilitating unified hash learning. Figure \ref{method} illustrates the PSCA framework, and the primary contributions are:
\begin{itemize}
	\item {\texttt{}} An orthogonal prototype learning method is proposed that achieves effective class-level semantic alignment instead of pursuing pair-wise sample alignment.
	\item {\texttt{}} A semantic consistency alignment is designed to dynamically correct unreliable pseudo-labels by combining geometric proximity with semantic predictions, addressing the semantic error accumulation problem.
	\item {\texttt{}} A feature reconstruction strategy leverages discriminative prototypes and membership matrix to create enhanced feature representations, ensuring hash quantization on reliable rather than noisy information.
	\item {\texttt{}} Comprehensive experiments demonstrate that PSCA outperforms the existing state-of-the-art DAR methods.
\end{itemize}
\begin{figure*}[t]
	\begin{center}
		\includegraphics[scale=0.58]{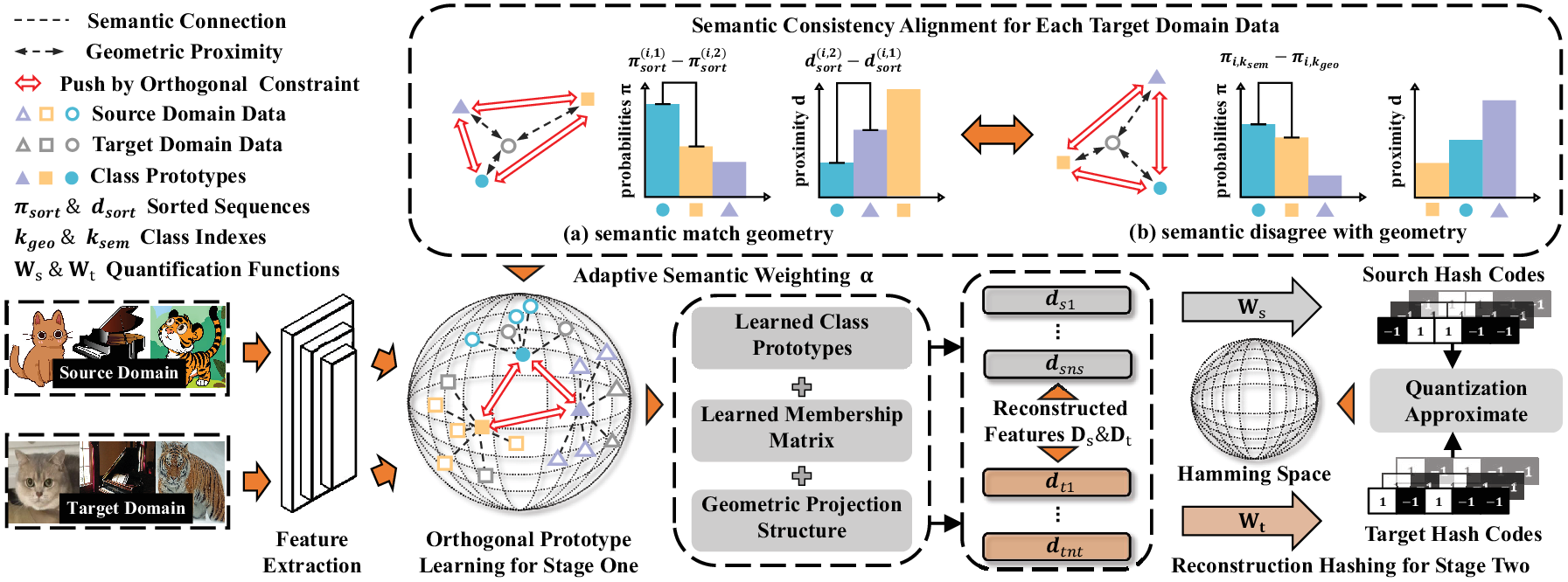}\\
		\caption{The framework of the proposed PSCA}
		\label{method}
	\end{center}
\end{figure*}
\section{Related Work}
\subsection{Hash Learning}
Hashing seeks to encode high-dimensional data into low-dimensional binary codes with maintaining similarity relationships, offering the merits of low storage demand and high retrieval efficiency \cite{hu2025consistent}. Hashing can be divided into two categories according to whether the semantic information is available, i.e., unsupervised hashing and supervised hashing. Supervised methods like DCMVH \cite{zhu2020deep} and MSDH \cite{luo2019discrete} have achieved noteworthy performance. Nevertheless, the significant cost related to annotating labels brings a major obstacle to their scalability. By contrast, benefiting from no need for explicit semantics, unsupervised hashing methods serve as the more suitable retrieval manner for real-world applications. Some representative methods among this category include SGH \cite{jiang2015scalable} and GraphBit \cite{wang2023learning}. However, unsupervised algorithms may pose a potential risk of limiting the discriminative power of hash codes due to the lack of explicit semantics. While admiring the positive performance of these methods, they are limited for DAR scenarios due to their assumption of retrieving in a single domain.
\subsection{Domain Adaptive Retrieval}
DAR assumes that there exists a well-labeled source domain and an unlabeled target domain, where these two domains are related yet different in distributions. It aims to mitigate the domain discrepancies and preserving semantic similarity in the learned hash codes, thereby achieving effective retrieval. PWCF \cite{huang2020probability} and DAPH \cite{huang2021domain} geometrically constrain the distribution of samples to bridge domain gap. Since two domains share a common semantic space, TSS, SGHL and DCS-LSG further align pseudo-labels to target domain samples to encourage the transfer of domain-specific valuable insights. Note that aforementioned methods and the proposed PSCA are on the basis of machine learning principles. 

With the rise of deep learning, several advanced deep learning-based DAR methods have emerged recently. PEACE \cite{wang2023toward} considers the uncertainty of pseudo-labels and progressively boosts their reliability. CPH \cite{cui2024effective} innovatively constructs a domain-shared unit subspace, then aligns domains through prototype contrastive learning. COUPLE \cite{luo2025cross} employs graph flow diffusion to simulate the cross-domain knowledge transfer, dynamically identifying lower noise clusters.

\section{Proposed Method}
\subsection{Problem Definition}
Assume we have a source domain ${\mathcal{D}_{s}=\{\mathbf{x}_{s_i},\mathbf{y}_{s_i}\}_{i=1}^{n_{s}}}$ comprising ${n_s}$ samples, a target domain ${\mathcal{D}_t=\{\mathbf{{{x}}}_{t_i}\}_{i=1}^{n_t}}$ including ${n_t}$ samples. Let ${\mathbf{X}_{s}\in\mathbb{R}^{d\times n_{s}}}$ and ${\mathbf{{{X}}}_{t}\in\mathbb{R}^{d\times n_{t}}}$ as the original sample matrices of two domains respectively, where ${d}$ denotes the feature dimension. We define ${\mathbf{X}=[\mathbf{X}_{s},\mathbf{{{X}}}_{t}]\in\mathbb{R}^{d\times n}}$ as the total sample matrix, where $n = n_s +n_t$. For each target sample $\mathbf{x}_{t_i}$, we obtain pseudo-label probabilities $\boldsymbol{\pi}_i = [\pi_{i1}, \pi_{i2}, \ldots, \pi_{ic}] \in \mathbb{R}^{1 \times c}$ using Nearest Class Prototype (NCP) and Structured Prediction (SP) \cite{wang2020unsupervised}, where $c$ is the number of categories. Let $\pi_{\text{sort}}^{(i,e)}$ be the $e$-th largest element in $\boldsymbol{\pi}_i$, where $\pi_{\text{sort}}^{(i,1)} \geq \pi_{\text{sort}}^{(i,2)} \geq \cdots \geq \pi_{\text{sort}}^{(i,c)}$. Then the pseudo-label is assigned according to $\pi_{\text{sort}}^{(i,1)}$. The detailed pseudo-labeling process is presented in the supplementary materials. Define the overall label set as ${\mathbf{Y}=[\mathbf{Y}_{s},\mathbf{\hat{Y}}_{t}]^\top}\in\{0,1\}^{n\times c}$, where $\mathbf{\hat{Y}}_{t}\in\mathbb{R}^{n_t\times c}$ are the pseudo-labels. The ultimate objective of DAR is to learn a set of similarity-preserving hash codes ${\mathbf{B}=[\mathbf{b}_{{s_1}},\cdots,\mathbf{b}_{{s_n}_s},\mathbf{b}_{{t_1}},\cdots,\mathbf{b}_{{t_n}_t}]\in\{-1,1\}^{r\times n}}$ for effective retrieval, where $r$ represents the hash code length.

\subsection{Prototype-Based Semantic Consistency Alignment}
Due to domain shift, cross-domain samples with similar semantics may exhibit significantly distinct distributions. To bridge this gap, we first employ Maximum Mean Discrepancy (MMD) to align domain marginal distributions:
\begin{equation}
	\begin{aligned}
		\label{marginal_distribution}
		\min_{\mathbf{P}}
		&{\left\|\frac{1}{{n}_s}\sum_{i=1}^{{n}_s}\mathbf{P}^\top{\mathbf{x}_i}-\frac{1}{{n}_t}\sum_{{q}={n}_s+1}^{{n}_s+{n}_t}\mathbf{P}^\top{\mathbf{{x}}_q}\right\|_2^2}
		\\&{=\mathrm{Tr}({\mathbf{P}^\top}\mathbf{X}\mathbf{H}{\mathbf{X}^\top}\mathbf{P})}
	\end{aligned}
\end{equation}
where $\mathbf{P}\in\mathbb{R}^{d\times q} $ is the projection matrix that maps original features into a common $q$-dimensional subspace ($q<<d$), where the domain gap is expected to be bridged. $\mathbf{H}$ is the MMD centering matrix defined as:
\begin{equation}
	\quad {{h}_{iq}}=\left\{ \begin{aligned}
		\label{F_2}
		&{\frac{1}{n_sn_s},\quad \mathbf{x}_i,\mathbf{x}_q\in\mathcal{D}_{s}} \\
		&{\frac{-1}{n_sn_t},\quad \mathbf{x}_i\in\mathcal{D}_{s}} \wedge\mathbf{{x}}_q\in\mathcal{D}_{t}\\
		&{\frac{1}{n_tn_t},\quad {\mathbf{{x}}_i},{\mathbf{{x}}_q}\in\mathcal{D}_{t}}\\
	\end{aligned} \right.
\end{equation}

Apart from marginal discrepancies, semantic structure within data distributions requires consideration. Existing methods achieve this by aligning conditional distributions or semantically consistent sample pairs, yet suffer from computational inefficiency and sensitivity to outliers.

To circumvent these issues, we introduce a prototype-based approach that learns discriminative class centers ${\mathbf{O}=[\mathbf{o}_{1},\cdots,\mathbf{o}_{c}]\in\mathbb{R}^{{q}\times c}}$, directly modeling class-level semantic connections to facilitate effective class alignment:
\begin{equation}
	\begin{aligned}
		\label{central_similarity}
		\min_{\mathbf{P},\mathbf{O}}\sum_{i=1}^{n}\sum_{j=1}^{c}y_{ij}\|\mathbf{P}^{\top}\mathbf{x}_{i}-\mathbf{o}_{j}\|_{2}^{2},\quad
		{\text{s.t.}~\mathbf{O}^{\top}\mathbf{O}=\mathbf{I}_{c}},
	\end{aligned}
\end{equation}
where $\mathbf{I}_c$ represents a $c$-dimensional identity matrix. Ideally, $\mathbf{O}$ can effectively reduce the intra-class divergence by gathering samples with identical semantics. Meanwhile, $\mathbf{O}^{\top}\mathbf{O}=\mathbf{I}_{c}$ ensures that different prototypes are maximally separated by forcing the inner product between distinct prototypes to be zero. Nevertheless, Eq. \eqref{central_similarity} exhibits two critical limitations: 1) incorrect $\mathbf{\hat{y}}_{t}$ shifts these domain-shared prototypes away from their latent ground-truth positions; 2) binary pseudo-labels fail to convey prediction confidence, hampering assessment of assignment reliability.

Addressing the above challenges, we propose a semantic consistency alignment that tackles them simultaneously. Since samples sharing the identical prototype have consistent labels, samples closer to prototypes are inherently more reliable for prediction than those at cluster boundaries. Therefore, $\mathbf{R}\in\mathbb{R}^{n_t\times c}$ is designed as a soft membership matrix, where each element $r_{ij}$ provides a more reliable membership degree than ${\hat{y}}_{ij}$. This is achieved by incorporating the geometric proximity $d_{ij}=\|\mathbf{P}^{\top}\mathbf{x}_{t_i}-\mathbf{o}_{j}\|_{2}^{2}$. In particular, $\mathbf{R}$ is optimized via the following objective:
\begin{equation}
	\begin{aligned}
		\label{R_obtain}
		\min_{\mathbf{R} \geq \mathbf{0}, \mathbf{R}\mathbf{1}_c = \mathbf{1}_{n_t}}\sum_{i=1}^{n_t}\sum_{j=1}^{c}\{{{r}}_{ij}^{\sigma}d_{ij}-{\psi}_{ij}\text{log}({r}_{ij})\}
	\end{aligned}
\end{equation}
where $\mathbf{1}$ denotes all-ones vector with dimension indicated by its subscript. In the first term, ${{r}}_{ij}$ measures geometric proximity of projected target samples to prototypes. The coefficient $\sigma>1$ amplifies the penalty differences as ${{r}_{ij}}^{\sigma}<{r}_{ij}$, creating non-linear down-weighting. The second term serves as a semantic-aware term where $\psi$ are elements of an adaptive weighting matrix $\boldsymbol\Psi = \boldsymbol{\alpha} \odot \hat{\mathbf{Y}_t}$, with $\boldsymbol{\alpha}=[\alpha_{1},\cdots,\alpha_{n_t}]^\top \in \mathbb{R}^{n_t \times 1}$ broadcasted to match dimensions. This term activates only when $\hat{y}_{ij} = 1$, prompting larger ${r}_{ij}$ values for the semantically predicted class. Here, $\alpha_{i}$ adaptively adjusts the semantic fusion strength by: 
\begin{equation}
	\alpha_i=\left\{ 
	\label{Alpha}
	\begin{aligned}
		&\frac{\pi_{\text{sort}}^{(i,1)} - \pi_{\text{sort}}^{(i,2)}}{d_{\text{sort}}^{(i,2)} - d_{\text{sort}}^{(i,1)}+\text{eps} }, &\text{if } {k}_{\text{geo}} = {k}_{\text{sem}}  \\
		&\pi_{\text{sort}}^{(i,1)}(1 - |\pi_{i,{k}_{\text{geo}}} - \pi_{i,{k}_{\text{sem}}}|), &\text{otherwise}
	\end{aligned} \right.
\end{equation}
where $\text{eps}$ prevents the situation of division by zero. Let $\mathbf{d}_i = [d_{i1}, d_{i2}, \ldots, d_{ic}] \in \mathbb{R}^{1 \times c}$ denote the distances from $\mathbf{x}_{t_i}$ to all prototypes, we have the sorted sequence as $d_{\text{sort}}^{(i,1)} \leq \cdots \leq d_{\text{sort}}^{(i,c)}$, where $d_{\text{sort}}^{(i,1)}$ indicates the smallest value in $\mathbf{d}_i$. $k_{\text{geo}} = \arg\min_j d_{ij}$ and $k_{\text{sem}} = \arg\max_j \pi_{ij}$ denote the indexs of geometrically closest and semantically preferred prototypes. When  geometric and semantic indicators are consistent ($k_{\text{geo}} = k_{\text{sem}}$), larger semantic margins encourage greater reliance on semantic information by increasing $\alpha_i$, while larger geometric margins favor geometric knowledge by decreasing $\alpha_i$. In the conflicting case, $\alpha_i$ is reduced based on the disagreement intensity, i.e., $|\pi_{i,{k}_{\text{geo}}} - \pi_{i,{k}_{\text{sem}}}|$, thereby decreasing reliance on potentially erroneous assignments.

With the optimized $\mathbf{R}$ from Eq. \eqref{R_obtain}, we form the unified semantic matrix as ${\mathbf{\widetilde{Y}}}=[\mathbf{Y}_{s},\mathbf{R}]^\top$. Integrating all key components above, stage one is formulated as:
\begin{equation}
	\begin{aligned}
		\label{Stage_One}
		&\min_{\mathbf{P},\mathbf{O}^{\top}\mathbf{O}=\mathbf{I}_{c}}\sum_{i=1}^{n}\sum_{j=1}^{c}{\widetilde{y}}_{ij}\|\mathbf{P}^{\top}\mathbf{x}_{i}-\mathbf{o}_{j}\|_{2}^{2}\\&
		+{\lambda_1\mathrm{Tr}({\mathbf{P}^\top}\mathbf{X}\mathbf{H}{\mathbf{X}^\top}\mathbf{P})}+\lambda_2\|\mathbf{P}\|_{2,1}
	\end{aligned}
\end{equation}
where $\lambda_1$ and $\lambda_2$ are the trade-off parameters. Here, $\|\mathbf{P}\|_{2,1}$ serves as a $\ell_{2,1}$-norm with row-sparsity constraint.

\subsection{Feature Reconstruction Hashing}
Given that projection $\mathbf{P}$ focuses on domain alignment, directly quantizing $\mathbf{P}^\top\mathbf{X}$ neglects the semantic discriminability established by prototype learning, inevitably reducing the hash coding quality. Thus, as an intermediate operation to avoid prematurely using features lacking semantic enhancement and seamlessly bridge both stages, we reconstruct semantically enhanced features by leveraging the reliable prototypes $\mathbf{O}$ and learned memberships $\mathbf{R}$ from stage one.

Since each row $\mathbf{r}_{i}=[{r}_{i1},{r}_{i2},...,{r}_{ic}]$ in $\mathbf{R}$ indicates the confidences of $\mathbf{x}_{t_i}$ belonging to each category, and the learned class prototypes $\mathbf{O}$ exhibit stronger discriminability than the original data due to the orthogonal constraint and semantic consistency alignment. Consequently, we reconstruct target samples $\mathbf{{x}}_{t_i}$ as $\mathbf{\widetilde{x}}_{t_i} = \sum_{m}^{c}{r}_{im}\mathbf{o}_{m}^\top$. This confidence-weighted combination of prototypes encodes semantically enhanced representations. Meanwhile, the reliable labels are embedded into the source reconstruction, i.e., $\mathbf{\widetilde{x}}_{s_i} = \sum_{m}^{c}{y}_{im}\mathbf{o}_{m}^\top$. This yields the novel representation $\mathbf{\widetilde{X}}=[\mathbf{\widetilde{x}}_{s_1},...,\mathbf{\widetilde{x}}_{s_{ns}},\mathbf{\widetilde{x}}_{t_1},...,\mathbf{\widetilde{x}}_{t_{nt}}]^\top\in\mathbb{R}^{n\times q}$ that better reveals the underlying ground-truth semantic structure. However, while $\mathbf{\widetilde{X}}$ excels in semantic discriminability, it may lose certain geometric structures of the original data, making sole reliance on $\mathbf{\widetilde{X}}$ insufficient for comprehensive representation during hash coding. Hence, we fuse $\mathbf{\widetilde{X}}$ and the projected features $\mathbf{P}^\top\mathbf{X}$. Let $\mathcal{C}=2q$, the overall reconstructed features are expressed as $\mathbf{D}=[{\mathbf{D}_{s}},{\mathbf{D}_{t}}]=[{\mathbf{d}_{s}}_1,...,{\mathbf{d}_{s}}_{ns}, {\mathbf{d}_{t}}_1,...,{\mathbf{d}_{t}}_{nt}]=[\mathbf{\widetilde{X}},\mathbf{X}^\top\mathbf{P}]^\top\in\mathbb{R}^{\mathcal{C}\times n}$. This concatenation both retains the reconstructed semantics and geometric projection structures.
Building upon the reconstructed features $\mathbf{D}$, we proceed to develop discriminative hash codes in stage two. Towards further mitigating the adverse effects of pseudo-labeling and domain discrepancies, meanwhile for better transferring domain-specific information, we design distinct functions for both domains, i.e., $\mathbf{W}_s\in\mathbb{R}^{r\times \mathcal{C}}$ and $\mathbf{W}_t\in\mathbb{R}^{r\times \mathcal{C}}$. Given our objective is to learn a unified Hamming space instead of two separate domain-specific hash codes. Hence, the dual functions are required to be approximated during hash learning. Ultimately, the objective of stage two is given below:
\begin{equation}
	\begin{aligned}
		\label{Stage_Two}
		&\begin{aligned}
			\min_{\mathbf{W}_s,\mathbf{W}_t,\mathbf{B}_s,\mathbf{B}_t}&\|\mathbf{W}_s\mathbf{D}_{s}-\mathbf{B}_{s}\|_{F}^{2}+\|\mathbf{W}_t\mathbf{D}_{t}-\mathbf{B}_{t}\|_{F}^{2}\\&
			+\lambda_3\|\mathbf{W}_s-\mathbf{W}_{t}\|_{F}^{2}
		\end{aligned}\\
		&\text{s.t.}~\mathbf{W}_s{\mathbf{W}_s}^{\top}=\mathbf{I}_{{r}},\mathbf{W}_t{\mathbf{W}_t}^{\top}=\mathbf{I}_{{r}},\mathbf{B}_s,\mathbf{B}_t \in \{-1,1 \}^{r\times *},
	\end{aligned}
\end{equation}
where the orthogonal constraints ensure independence hash bit encoding. $\lambda_3$ is a trade-off parameter and $*\in \{n_s,n_t \} $.

\begin{table*}[t]
	\centering
	\fontsize{9}{9.2}\selectfont
	\setlength{\tabcolsep}{1.4mm}
	\begin{tabular}{c|cccc|cccc|cccc|cccc}
		\hline
		Case&
		\multicolumn{4}{c|}{MNIST$\to$USPS} &
		\multicolumn{4}{c|}{COIL1$\to$COIL2} &
		\multicolumn{4}{c|}{A$\to$D}&\multicolumn{4}{c}{A$\to$W} \\ \hline
		Code length & 16 & 32  & 64  & 128 & 16 & 32  & 64  & 128 & 16 & 32  & 64  & 128 & 16 & 32  & 64 & 128\\ \hline
		SH & 15.56 & 13.67  & 13.54 & 12.95 & 40.18 &44.64  & 42.84& 38.36& 14.08& 13.62& 12.02& 10.91 & 12.04 & 11.97 & 9.83 & 9.90  \\
		DSH& 20.60& 22.21& 24.28& 26.50& 37.92& 44.85& 46.38& 46.05& 11.48& 13.86& 16.66& 19.88 & 9.58 & 12.14 &  15.09& 18.05\\
		SGH& 14.24& 16.69& 19.70& 21.95& 51.04& 52.31& 51.77& 50.53& 19.92& 21.19& 24.86& 27.50 & 16.95  & 20.13 & 22.47 & 25.44\\
		OCH& 13.73 &17.22 &20.18  &23.34 &46.50 &50.67  &55.25 &56.54 &14.29 &20.43 &24.86  &27.50 &14.85 & 20.24 &  22.49 & 25.86\\
		ITQ+& 22.84 &21.20  &19.15 & 18.52 &46.68 &50.49 &50.80  &50.63 &15.42 &16.74  &17.99  &16.59 & 14.94 & 16.19 & 15.00 & 15.21\\
		LapITQ+& 24.26 &24.03  &24.59 & 22.73 &44.44 &39.26 &34.20 &29.07 &17.53 &19.38& 19.96 & 18.12 & 15.10 & 17.80 & 18.24 &16.36 \\
		GTH-h& 19.47 &16.52 &15.33 &16.46 &44.89 &50.47&52.83  &52.93 &14.42 &20.38 &23.13 &24.08 & 13.55 &19.89  & 22.21  & 23.32\\
		PWCF& 43.90 &50.94 &52.51 & 53.17 &65.29 &64.98  &67.34 &67.00 &24.53 &29.57 &32.46 &34.55 & 21.98 &32.38 & 34.14 &35.21 \\
		DAPH*& 30.22 &35.14 &38.18 &39.36 &75.77 &77.58 &78.49&81.96& 29.14 &29.69 &28.17 &27.83 & 21.22 & 25.36 & 26.85 &28.74 \\
		SGHL& 62.95& 65.93 & 69.52 &71.46 &70.81 &78.71 &80.59 &83.00 &43.92 &51.55 &54.89 &59.91 & 44.49 &52.48  & 55.31 &55.64 \\
		TSS& \underline{64.19} &\underline{69.11} &\underline{72.59} &\underline{73.88} &78.07 &82.08  &\underline{85.27}& \underline{87.55} &17.83 &33.41 &44.86  &45.23 & 25.59 & 39.77 & 48.97 &53.23 \\
		DCS-LSG& 48.83 &53.31 &54.22 &59.88 &\underline{82.44} &\underline{83.36}  &84.38 &85.70& \underline{55.13}& \underline{58.82} &\underline{63.32} &\underline{64.59} & \underline{48.89} & \underline{53.90} & \underline{56.33} &\underline{57.13} \\\hline
		Ours& \textbf{86.05}& \textbf{86.47}& \textbf{87.35}& \textbf{88.71}  & \textbf{84.74}& \textbf{87.36}& \textbf{88.79}& \textbf{90.76}& \textbf{56.44}& \textbf{65.51}& \textbf{68.85}& \textbf{67.41} & \textbf{56.72} & \textbf{60.86} & \textbf{62.13} &  \textbf{65.78}\\ \hline 
	\end{tabular}
	\caption{Cross-domain retrieval performance (MAP\%) on MNIST$\rightarrow$USPS, COIL1$\rightarrow$COIL2, A$\rightarrow$D and A$\rightarrow$W with varying code lengths. The bolded figures indicate the highest scores, and underlined figures indicate the second-highest scores.}
	\label{table1}
\end{table*}

\section{Optimization}
\subsection{Solution Process}
As shown in Eqs. \eqref{Stage_One} and \eqref{Stage_Two}, there exist multiple variables that require to be solved, i.e., $\mathbf{P}$, $\mathbf{O}$, $\mathbf{R}$, $\mathbf{W}_s$, $\mathbf{W}_t$, $\mathbf{B}_s$ and $\mathbf{B}_t$. Due to space limitations, the detailed solution process of PSCA is presented in the supplementary materials. Meanwhile, an algorithm analysis subsection is also provided, including the computational complexity, the convergence analysis and running time comparison.

\subsection{Out-of-Sample Extension}
To facilitate generalization of unseen samples, a linear transformation matrix $\mathbf{\Phi}$ is derived to model the regression mapping between features and learned binary hash codes: 
\begin{equation}
	\begin{aligned}
		\label{HASH_function}
		\min_{\mathbf{\Phi}}\|\mathbf{\Phi}\mathbf{X}-\mathbf{B}\|_{F}^{2}+\beta\|\mathbf{\Phi}\|_{F}^2
	\end{aligned}
\end{equation}
where $\|\mathbf{\Phi}\|_{F}^2$ denotes a regularization term and $\beta$ serves as a balancing parameter. Eq. \eqref{HASH_function} can be easily solved with:
\begin{equation}
	\begin{aligned}
		\label{HASH_function_s}
		\mathbf{\Phi} = \mathbf{B}\mathbf{X}^\top(\mathbf{X}\mathbf{X}^\top+\beta\mathbf{I})^{-1}
	\end{aligned}
\end{equation}
Thereafter, the corresponding binary hash code of any query sample can be obtained by $\mathbf{b}_{query}=\text{sgn}(\mathbf{\Phi}\mathbf{x}_{query})$.

\section{Experiment}
\subsection{Dataset and Evaluation Metric}
We carry out extensive comparative experiments on four public benchmark datasets, namely Office-31 \cite{saenko2010adapting}, Office-Home \cite{venkateswara2017deep}, COIL20 \cite{nene1996columbia}, and MNIST-USPS \cite{lecun1998gradient, hull1994database}. \textbf{Office-31} contains 4,110 images shot in office environments across three domains: Amazon (A), Webcam (W), and DSLR (D). For experiments, we establish two transferable retrieval cases: A$\to$D and A$\to$W. \textbf{Office-Home} consists of images across 65 categories found in office and domestic environments, categorized into four domains: the Artistic (Ar) with 2,427 samples, the Clipart (Cl) with 4,365 samples, the Product (Pr) with 4,439 samples, and Real-world (Rw) with 4,357 samples. Consistent with previous baselines, we select six domain permutation cases: R$\rightarrow$P, R$\to$C, A$\to$R, P$\to$R, C$\to$R, and R$\to$A. \textbf{COIL20} includes 1,440 images of 20 different objects. COIL1 and COIL2 are two subsets, each containing images captured from diverse angles. According to \cite{long2014transfer}, a 1,024-dimensional feature vector is extracted for per image, then we construct a retrieval case: COIL1$\to$COIL2. \textbf{MNIST-USPS} are two handwritten datasets including digit images from 0 to 9. Following \cite{long2013transfer}, we resize MNIST to 16$\times$16 pixels and create MNIST-USPS dataset by picking 2,000 MNIST and 1,800 USPS images. We create a retrieval case for experiments: MNIST$\to$USPS.

Random 10\% of target domain samples are selected as the testing set, while the remaining 90\%, along with entire source samples, form the training set. For cross-domain retrieval, queries are expected to match the most similar images in source domain, whereas for single-domain retrieval, target domain is regarded as the retrieved database.

The mean Average Precision (MAP), Top-K Precision Curve, and Precision-Recall Curve are used to measure the hash coding quality. Note that higher values denote better performance for all evaluation metrics. Each trial is repeated 10 times, and we report the average MAP scores (\%). 

\begin{table*}[th]
	\centering   
	\fontsize{9}{9.2}\selectfont
	\setlength{\tabcolsep}{0.9mm}{
		\begin{tabular}{c|ccc|ccc|ccc|ccc|ccc|ccc}
			\hline
			Case    &\multicolumn{3}{c|}{P$\to$R}  & \multicolumn{3}{c|}{C$\to$R} & \multicolumn{3}{c|}{R$\to$A} & \multicolumn{3}{c|}{R$\to$P} & \multicolumn{3}{c|}{R$\to$C} & \multicolumn{3}{c}{A$\to$R} \\ 
			\hline
			Code length  &16& 64&128   &16& 64&128    &16& 64&128   &16& 64&128 &16& 64&128 &16& 64&128 \\ \hline
			SH   & 10.96&15.03 &14.08 & 6.27 &8.77 &7.97 &9.47 &12.87 &11.62 & 11.37 &16.13&15.08 & 5.75 &8.24  &7.68  &10.28 &13.71 &12.30    \\
			DSH   & 5.61 &8.49 &9.79 &3.57 &5.47 &6.55 &5.43 &9.67 &10.54 &5.70 &8.26 &10.20 &3.62 & 5.28 &6.29 &5.95 & 9.69  &11.52  \\
			SGH   &16.68  &24.51 &26.38 &7.22 &13.62 &14.82 &11.92 &22.53 &24.69 &15.85 & 25.73 &27.89 &7.05 &13.51 &14.83 &13.32 &22.93 &25.14  \\
			OCH   &11.52 &18.65 &20.98 &6.15 &10.27 &11.21 &9.45 &17.54 &19.81 &11.18 &20.15 &22.27 &5.95 &10.05 &11.46 &10.30 &18.09  &20.65  \\
			ITQ+  &11.25  &17.61 &17.74 &6.58 &9.55 &9.34 &9.41 &14.25 &15.53 &- &- &-&- &- &-&- &-  &- \\
			LapITQ+  &11.99 &16.89 &16.02 &7.27 &10.37 &10.87 &8.89 &13.56 &13.75 &- &- &-&- &-&-& - &-  & -  \\
			GTH-h   &10.68 &19.80 &22.44 &6.70  &11.41 &12.69 &9.57 &17.54 &19.87 &12.01 &22.21 &23.94 &5.97 &11.63 &13.08 &11.00 &19.94  &21.24  \\
			PWCF   &21.41 &35.44 &35.85 &12.79 &21.97 &10.39 &22.57 &32.20 &31.25 &21.21 &35.51 &35.38 &13.79 &21.96 &20.67 &22.02 &32.63  &31.25  \\
			DAPH*   &30.45 &48.77 &44.93 &12.83 &23.45 &23.87 &29.69 &45.77 &43.18 &25.33 &44.77 &43.32 &18.60 &30.78 &31.83 &26.01 &34.43 &37.97   \\
			SGHL    &27.92  &49.38 &53.72 &18.11 &29.96 &34.34 &22.27 &38.83 &42.81 &28.73 &49.89 &53.78 &16.08 &28.45 &30.37 &20.01 &35.49  &39.13  \\
			TSS     &9.52 &48.24 &61.41 &6.64 &39.59 &50.90 &10.74 &38.28 &49.11 &10.61 &50.65 &62.58 &5.96 &24.54 &\underline{32.60} &12.57 &47.72 &57.11   \\
			DCS-LSG  &\underline{47.76} & \underline{69.72} &\underline{70.21} &\underline{37.39} &\underline{57.36} &\underline{61.59} &\underline{33.42} & \underline{48.81} &\underline{50.06} &\underline{54.53} &\underline{68.00} &\underline{70.27} &\underline{19.28}& \underline{30.45} &31.48 &\underline{47.52} &\underline{67.81} &\underline{67.73} \\ 
			\hline
			Ours  & \textbf{55.84}  & \textbf{76.06} &\textbf{78.15} &\textbf{47.59} & \textbf{68.61} &\textbf{68.82} &\textbf{44.84} & \textbf{64.02} &\textbf{68.77} &\textbf{58.83} & \textbf{73.04} &\textbf{75.06} &\textbf{28.34} & \textbf{39.65} &\textbf{42.87} &\textbf{54.17} & \textbf{72.78}  &\textbf{74.81}  \\ 
			\hline             
		\end{tabular}
	}
	\caption{Cross-domain retrieval performance (MAP\%) on Office-Home with varying code lengths.}
	\label{table2}
\end{table*}
\subsection{Baseline and Implementation Detail}
We select several state-of-the-art methods as baselines for comparison with PSCA: SH \cite{weiss2008spectral}, DSH \cite{liu2016deep}, SGH \cite{jiang2015scalable}, ITQ+ \cite{zhou2018transfer}, LapITQ+ \cite{zhou2018transfer}, GTH-h \cite{zhang2019optimal}, PWCF \cite{huang2020probability}, DAPH* \cite{huang2021domain}, SGHL \cite{zhang2023semantic}, TSS \cite{chen2023two}, and DCS-LSG \cite{zhang2024dynamic}. Where SH, DSH, and SGH belong to traditional hashing. ITQ+, LapITQ+, and GTH-h are transfer hashing. PWCF, DAPH*, SGHL, TSS, and DCS-LSG are hashing methods which aim at dealing with cross-domain scenarios. Note that DAPH* is the supervised variant of DAPH. ITQ+ and LapITQ+ achieve effective retrieval only when the source domain contains more samples than the target domain. Thus in some cases, their MAP scores are not reported. Additionally, we further compare with three advanced deep DAR baselines, i.e., PEACE \cite{wang2023toward}, CPH \cite{cui2024effective} and COUPLE \cite{luo2025cross}.

The proposed PSCA includes three independent parameters, i.e., $\lambda_1$ to $\lambda_3$, each of them controls the penalty weight of different objectives. We empirically explore the optimal combination of them by fixing one and adjust other parameters: within the range of [1, 10, 100] for $\lambda_1$ and $\lambda_3$, range [0.1, 1, 10] for $\lambda_2$. The specific parameters depend on the characteristics of datasets. $\beta$ is fixed as 0.1 and $\sigma$ is set to 2.

\begin{figure}[t]
	\centering		
	\begin{center}
		\centering
		\subfigure
		{
			\includegraphics[width=0.175\textwidth,trim={9mm 0mm 15mm 1.5mm},clip]{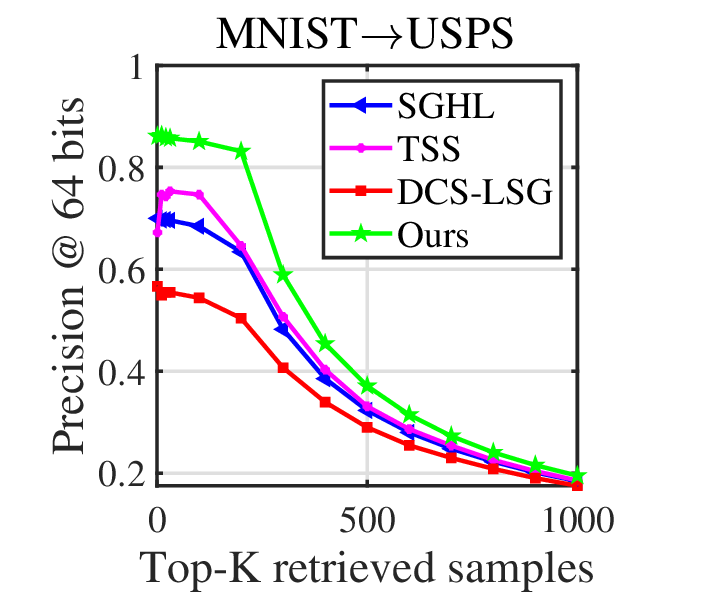}}
		{
			\includegraphics[width=0.175\textwidth,trim={9mm 0mm 15mm 1.5mm},clip]{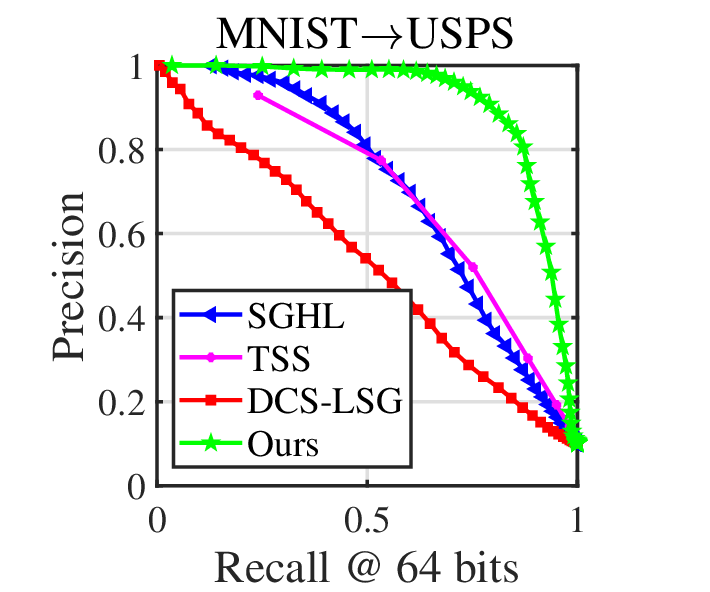}
			}
		\subfigure	
		{
			\includegraphics[width=0.175\textwidth,trim={9mm 0mm 15mm 1.5mm},clip]{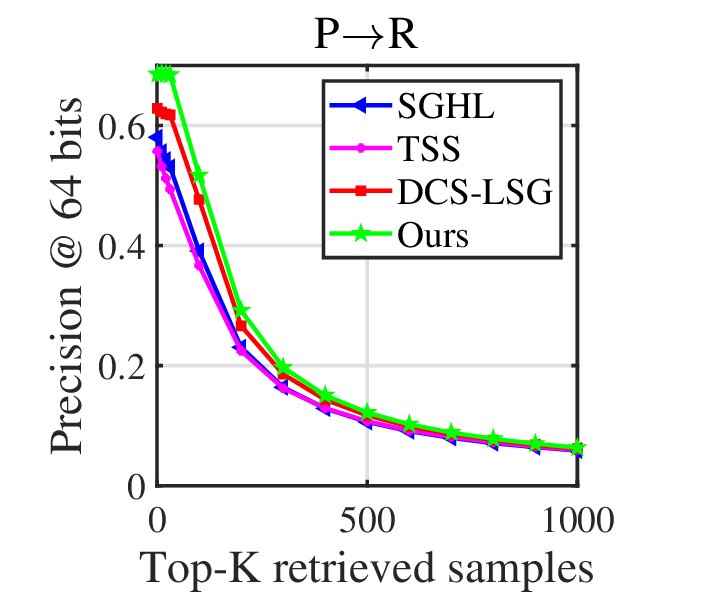}}
		{
			\includegraphics[width=0.175\textwidth,trim={9mm 0mm 15mm 1.5mm},clip]{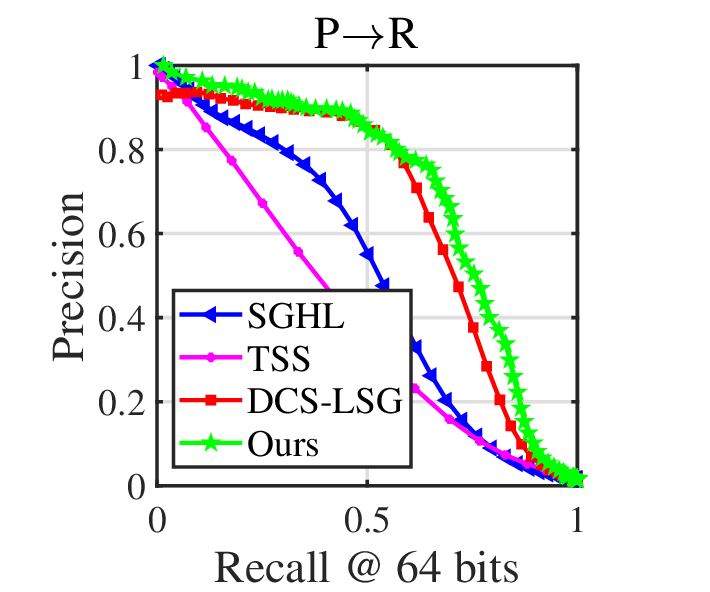}}
		\subfigure
		{
			\includegraphics[width=0.175\textwidth,trim={9mm 0mm 15mm 1.5mm},clip]{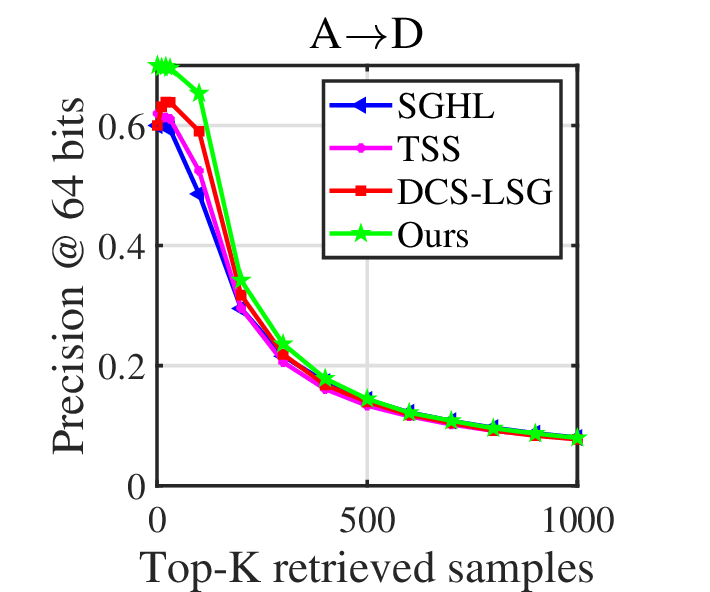}}
		{
			\includegraphics[width=0.175\textwidth,trim={9mm 0mm 15mm 1.5mm},clip]{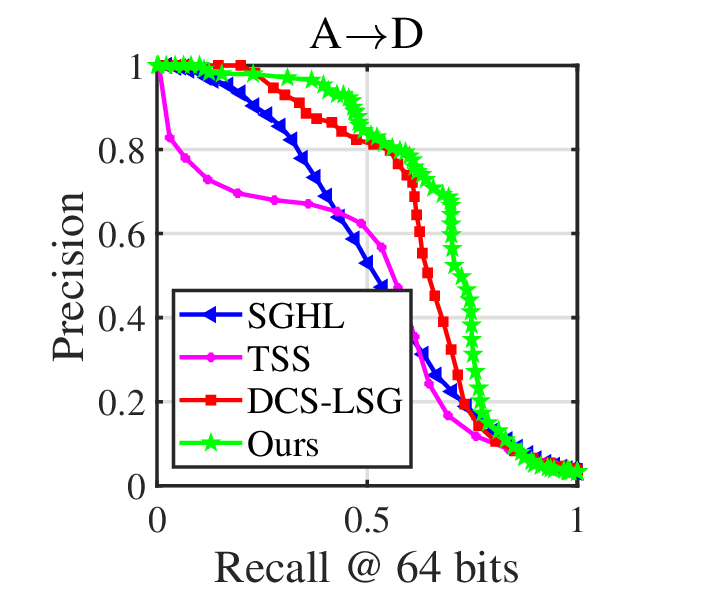}}
	\end{center}
	\caption{The Precision-Recall Curves and Top-K Precision Curves of SGHL, TSS, DCS-LSG and PSCA.}	\label{TOPK}
\end{figure}

\subsection{Experimental Analysis on Cross-Domain Retrieval}
To evaluate PSCA's cross-domain retrieval performance, we conduct comparison experiments with all baselines across varying code lengths. Tables \ref{table1} and \ref{table2} report the MAP results. By analyzing these tables, we derive the following observations: comprehensive experiments across four datasets demonstrate  that PSCA consistently outperforms all baselines across different code lengths. The presented average MAP of PSCA exceeds that of the second-best baselines by 17.21\%, 3.94\%, 4.08\% and 7.33\% on cases MNIST$\to$USPS, COIL1$\to$COIL2, A$\to$D and A$\to$W respectively. On Office-Home, PSCA achieves a remarkable performance improvement by 8.82\% on average across six retrieval cases. Based on the above analyses, PSCA exhibits consistent competitive advantages over other baselines. We conclude that PSCA performs well in handling cross-domain issues, whether on small-scale (MNIST-USPS, COIL20), medium-scale (Office-31), or large-scale datasets (Office-Home). To further evaluate the performance on cross-domain retrieval, we plot the Top-K Precision and Precision-Recall curves of PSCA and three most recent competitive methods, i.e., SGHL, TSS and DCS-LSG. Specifically, experiments are conducted on three retrieval tasks across datasets of different scales. As shown in Figure \ref{TOPK}, it can be observed from the Top-K Precision Curves that PSCA consistently maintains advantages as the number of retrieved samples increases. The Precision-Recall Curves demonstrate that PSCA outperforms the comparison baselines.

\subsection{Experimental Analysis on Single-Domain Retrieval}
To evaluate the effectiveness of PSCA in the realm of single-domain retrieval, four representative cases (MNIST$\to$USPS, COIL1$\to$COIL2, A$\to$D, P$\to$R) are selected for experiments. We can conclude the following observations by analyzing Table \ref{table3}: PSCA outperforms comparison baselines with the average MAP 6.81\%, 5.39\% and 12.55\% higher than the second-best baselines on MNIST$\to$USPS, COIL1$\to$COIL2 and P$\to$R respectively. Notably, PSCA shows the least performance improvement on A$\to$D by 2.25\%. This may be due to the fact that our domain-shared prototypes over-smooth reconstructed target domain features, hindering optimal single-domain retrieval.

\begin{table*}[t]
	\centering
	\fontsize{9}{9.2}\selectfont
	\setlength{\tabcolsep}{1.4mm}
	\begin{tabular}{c|cccc|cccc|cccc|cccc}
		\hline
		Case& \multicolumn{4}{c|}{MNIST$\rightarrow$USPS}   & \multicolumn{4}{c|}{COIL1$\rightarrow$COIL2} &
		\multicolumn{4}{c|}{A$\to$D}   &  \multicolumn{4}{c}{P$\to$R}                     \\ \hline
		Code length    & 16    & 32      & 64       & 128   & 16          & 32            & 64       & 128  &16          & 32          & 64     & 128 &16          & 32          & 64     & 128  \\ \hline
		SH      & 46.30 & 47.82  & 49.12 & 47.81 & 52.91          & 57.07 & 57.23 & 52.61 & 30.54 & 35.66  & 42.50  & 45.64 &13.15 &18.71 &22.57 &20.66 \\
		DSH     & 41.42 & 45.30& 47.85 & 50.76 & 43.44       & 52.85      & 58.06  & 58.84 & 22.45 & 33.38 & 40.09  & 46.31 &6.10 &11.44 &16.61 &14.45\\
		SGH     & 15.60 & 30.78 & 35.55  & 41.78 & 54.30   & 59.25 & 59.97 & 58.49 & 38.67 &45.59 & 53.57 & \underline{57.37} &18.97 &26.18 &32.61 &34.97\\
		OCH     & 24.23 & 32.90 & 36.34 & 44.36 & 54.24   & 61.08 & 65.56 & 65.98 & 33.30 & 41.65& 50.78 & 53.74 &13.45 &21.14 &25.34 &28.02\\
		ITQ+    & 50.22 & 49.66 & 44.38 & 43.21 & 58.74 & 60.53  & 61.86 & 60.94 & 35.03 & 42.62 & 43.12 & 39.12 &15.60 &20.60 &24.96 &24.05\\
		LapITQ+ & 54.19 & 55.24 & 55.77 & 54.08 & 53.05 & 48.90  & 40.92 & 34.58 & 37.60 & 42.91 & 44.55  & 38.87 &16.78 &22.26 &22.29 &21.85\\
		GTH-h    & 43.38 & 40.09 & 34.14 & 32.80 & 58.84 & 59.65 & 63.68 & 63.71 & 39.88 & 46.60 & 50.74 & 54.72  &13.37 &22.03 &26.40  &28.99\\
		PWCF    & 57.57 & 64.00 & 65.45 & 65.63 & 69.36 & 70.81 & 72.41 & 70.43 & 24.01 &31.25 &38.65 &43.35 &23.84 &33.83 &38.26 &37.91\\
		DAPH* &{63.95} &{70.53} &{72.38} &{73.00} &{71.93} &{74.87}&{75.35}  &{75.54} &
		{42.18} &
		{46.86} &
		{47.73} &
		{50.26} &20.43 &31.02 &33.30 &31.50\\
		SGHL&64.25 &69.64 & 70.97 & 71.22& 72.37&74.24  &76.67 &78.93&\underline{46.40} &\underline{53.66} &\underline{58.05}  &{57.30} &16.13 &29.04 &36.49 &40.75\\
		TSS &
		\underline{69.55} &
		\underline{74.15} &
		\underline{77.26} &
		\underline {78.13} &
		67.20 &
		{76.63} &
		{76.53} &
		\underline{82.84} &
		{27.47} &
		{43.88} &
		{54.07} &
		{51.66} &15.67 &29.92 &45.07 &\underline{53.97}\\ 
		DCS-LSG& 56.91&63.30 &61.55 &62.06 &\underline{73.77} &\underline{77.07} &\underline{77.93} &79.77 &42.46 &47.26 &51.53 &51.12 &\underline{33.45} &\underline{44.13} &\underline{48.50} & 50.62\\\hline
		Ours &
		\textbf{80.61} &
		\textbf{81.09} &
		\textbf{81.53} &
		\textbf{83.07} &
		\textbf{78.90} &
		\textbf{80.69} &
		\textbf{83.54} &
		\textbf{86.00} &
		\textbf{48.83} &
		\textbf{55.91} &
		\textbf{59.54} &
		\textbf{60.11}  & \textbf{45.70}&\textbf{55.55}  &\textbf{60.13} &\textbf{62.85}\\
		\hline
	\end{tabular}
	\caption{Single-domain retrieval performance (MAP\%) on MNIST$\rightarrow$USPS, COIL1$\rightarrow$COIL2, A$\rightarrow$D and P$\rightarrow$R with varying code lengths. The bolded figure indicates the highest score, and underlined figure indicates the second-highest score.}
	\label{table3}
\end{table*}

\subsection{Comparison with Deep Baseline}
To further validate the superiority of PSCA, we compare it with three advanced deep learning-based DAR methods, i.e., PEACE \cite{wang2023toward}, CPH \cite{cui2024effective} and COUPLE \cite{luo2025cross}. The comparison results are illustrated in Figure \ref{deep}. As shown in Figure \ref{d1}, PSCA markedly outperforms other competitors on MNIST-USPS, boosting the performance by 15.89\% compared to the second-best method, COUPLE. Note that PSCA is based on traditional machine learning principles and adopts shallow features for this experiment. Figure \ref{d2} shows the comparison results on Office-Home. Here, the 4,096-d deep features utilized for PSCA are extracted from a pre-trained VGG-16 model \cite{huang2020probability}. On average across six cases, PSCA surpasses the suboptimal baseline CPH by 1.98\%. This improvement stems from PSCA's capacity to address multiple limitations of other methods: CPH neglects error accumulation caused by erroneous pseudo-labels. PEACE and COUPLE perform domain alignment within Hamming space, where dimension is typically much smaller than original data, resulting in limited semantic information preservation.

\begin{figure}[t]
	\centering
	\subfigure[MNIST-USPS]{\label{d1}
		\includegraphics[width=0.224\textwidth,trim={0.5mm 3.8mm 9.5mm 4.5mm},clip]{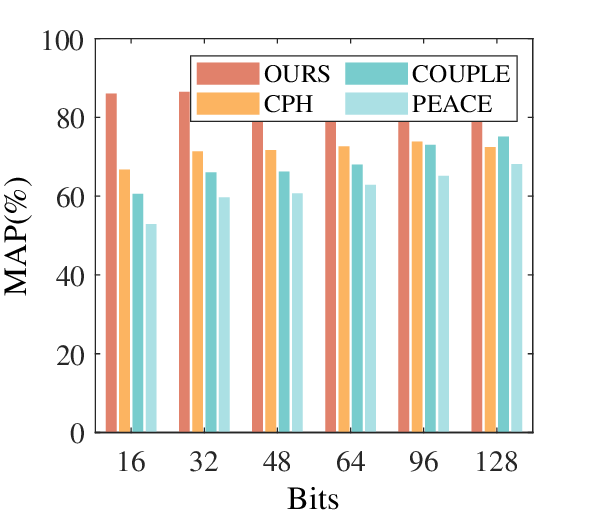}
		}
	\subfigure[Office-Home]{\label{d2}
	\includegraphics[width=0.224\textwidth,trim={0.5mm 3.8mm 9.5mm 4.5mm},clip]{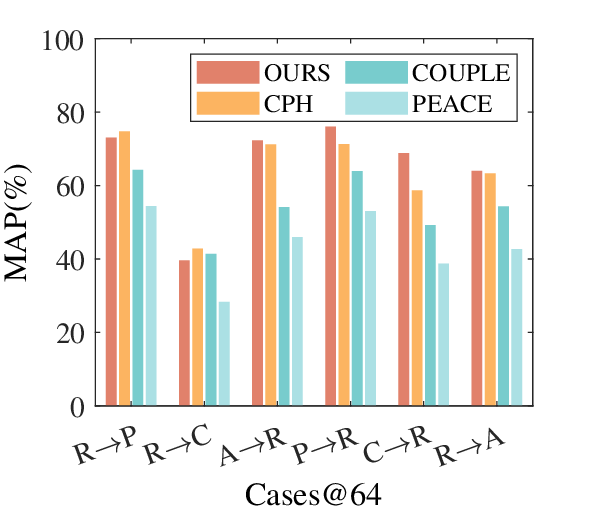}
		}
	\caption{Comparison with deep methods on MNIST-USPS across all code lengths and Office-Home with 64 bits.}
	\label{deep}
\end{figure}
\begin{table}[t]
	\centering
	\fontsize{9}{9.2}\selectfont
	\setlength{\tabcolsep}{1.4mm}
	{
		\begin{tabular}{c|cccccc}
			\hline
			Case&\multicolumn{6}{c}{MNIST$\rightarrow$USPS}  \\ \cline{1-7}
			Code length & 16 & 32 & 48 & 64 & 96 & 128   \\ \hline
			PSCA-v1      &  58.74 & 59.79 & 60.09  & 61.68 & 62.41 & 61.38 \\
			PSCA-v2    & \underline{77.11} & \underline{79.09}  &  \underline{81.89} & \underline{82.77} & \underline{82.95} &  \underline{83.24}     \\
			PSCA-v3     & 38.25& 40.26 & 42.08 & 43.19 & 43.40 & 44.56      \\
			PSCA-v4     & 76.35 & 78.76 & 79.31 & 79.87 & 80.10  & 80.92      \\\hline
			PSCA   & \textbf{86.05}& \textbf{86.47}& \textbf{87.04}& \textbf{87.35}&\textbf{88.09}& \textbf{88.71}   \\
			\hline
		\end{tabular} 
	}
	\caption{Ablation study results on MNIST-USPS.}
	\label{table4}
\end{table}

\subsection{Ablation Study}
To clearly highlight the significance of each component in PSCA, several variants are designed. \textbf{PSCA-v1} denotes that the semantic-aware fusion is removed, i.e., $\boldsymbol{\alpha}=\mathbf{0}$. This means that $\mathbf{R}$ is dominated by geometric structure knowledge. \textbf{PSCA-v2} denotes that we omit the semantic consistency alignment component, solely using Eq. \eqref{central_similarity} to conduct semantic alignment. For validity, we reconstruct target samples as $\mathbf{\widetilde{x}}_{t_i} = \sum_{m}^{c}{\hat{y}}_{im}\mathbf{o}_{m}^\top$.
\textbf{PSCA-v3} means that prototype learning is completely removed, subsequently we fuse samples as $\mathbf{D}_*=[\mathbf{X},\mathbf{X}^\top\mathbf{P}]^\top\in\mathbb{R}^{\mathcal{C}_*\times n}$, where $\mathcal{C}_* = d+q$. \textbf{PSCA-v4} indicates that feature reconstruction $\mathbf{D}$ is omitted, directly quantizing $\mathbf{P}^\top\mathbf{X}_s$ and $\mathbf{P}^\top\mathbf{{X}}_t$ in stage two. The results of ablation experiments are reported in Table \ref{table4}. It's evident that all components contribute to performance improvements of PSCA. The inferior performance of PSCA-v1 and PSCA-v3 demonstrates that prototypes indeed capture fundamental correct semantic patterns, validating the rationality of our geometric proximity for semantic correction.

\subsection{Visualization Analysis}
Figure \ref{TSNE} presents the t-SNE analyses of original data and hash codes generated by TSS, DCS-LSG and PSCA on MNIST$\rightarrow$USPS. As shown in Figure \ref{SNEa}, affected by domain gap, original data with the identical class generally distribute in two different areas. Compared to other baselines, PSCA achieves tighter intra-class clustering and clearer inter-class boundaries, demonstrating the effectiveness of our PSCA in generating more discriminative hash codes.

\begin{figure}[t]
	\centering
	\begin{center}
		\centering
		\subfigure[Original data]{
			\includegraphics[width=0.2\textwidth,trim={5mm 2.5mm 9.5mm 6mm},clip]{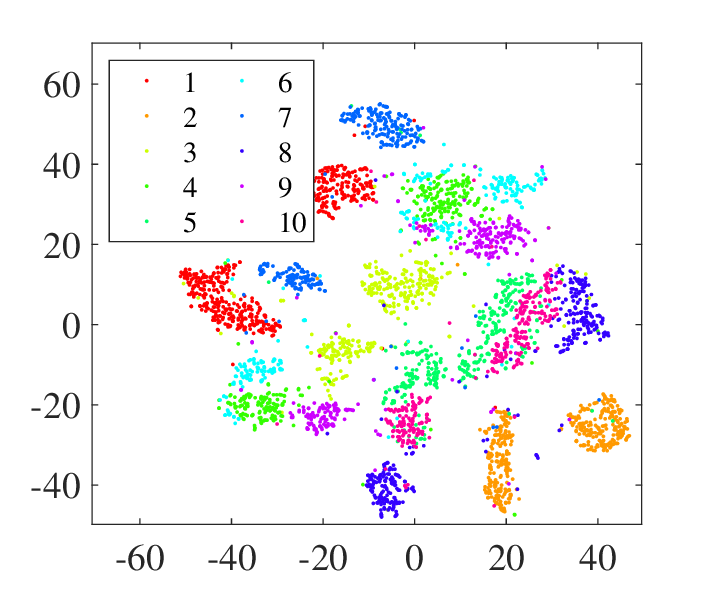}
			\label{SNEa}}
		\subfigure[TSS]{
			\includegraphics[width=0.2\textwidth,trim={5mm 2.5mm 9.5mm 6mm},clip]{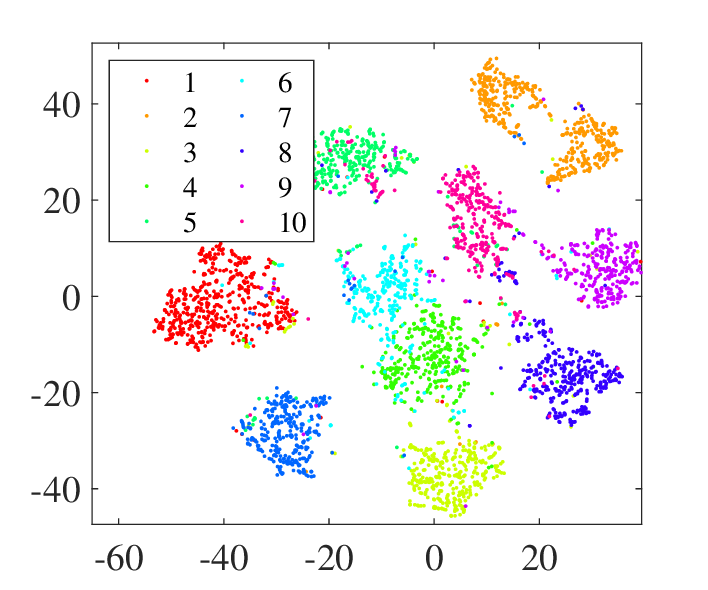}
			\label{SNEb}}
		\subfigure[DCS-LSG]{
			\includegraphics[width=0.2\textwidth,trim={5mm 2.5mm 9.5mm 6mm},clip]{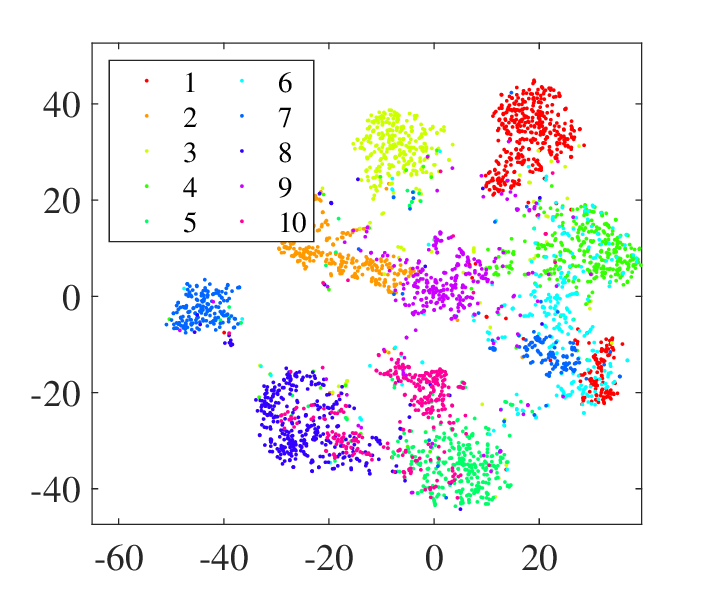}
			\label{SNEc}}
		\subfigure[PSCA]{
			\includegraphics[width=0.2\textwidth,trim={5mm 2.5mm 9.5mm 6mm},clip]{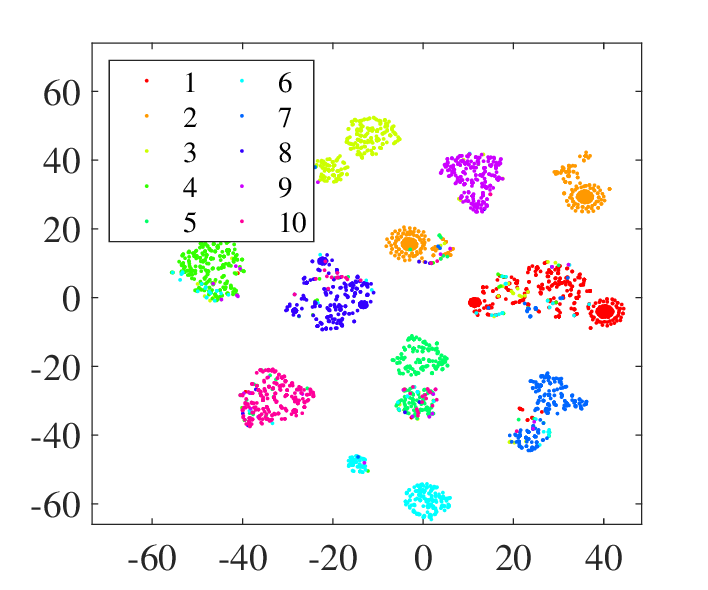}
			\label{SNEd}}
		\caption{Distribution visualization on case MNIST$\rightarrow$USPS. Different colors indicate different categories.}
		\label{TSNE}
	\end{center}
\end{figure}

\section{Conclusion}
While admiring the merits of previous methods, we identify their limitations and propose PSCA. Despite deviations in domain-shared class prototypes caused by incorrect pseudo-labels, they still capture the vast majority of accurate semantic patterns. Towards fully utilizing them, PSCA conducts orthogonal prototype learning and introduces the geometric proximity to weight potentially unreliable semantics via a novel semantic consistency alignment. By feature reconstruction, PSCA avoids directly quantizing original features affected by domain shift, achieving superior hash coding with remarkable retrieval performance. Comprehensive experiments validate that PSCA achieves SOTA performance.

\section{Acknowledgments}
This work was supported by the National Natural Science Foundation of China under Grant 62176065.

\bibliographystyle{aaai2026bst}
\bibliography{aaai2026}

\clearpage
\appendix
\appendixpage
\begin{abstract}
	The supplementary material firstly presents the detailed process of pseudo-labeling, secondly gives the optimization procedure of PSCA. Thirdly, an algorithm analysis subsection is provided, including convergence analysis, computational complexity and running time comparison. Finally, parameter sensitivity analysis is reported.
\end{abstract}

\section{Proposed Method}
\subsection{Pseudo-Labeling Strategy}
Pseudo-labeling serves as a widely adopted strategy in unsupervised hashing techniques \cite{chen2023two,zhang2024dynamic}. By inferring the latent correspondences between unlabeled samples and semantic categories, it enables the generation of comprehensively annotated datasets for downstream machine learning applications.

In this study, we assign pseudo-labels via Nearest Category Prototype (NCP) \cite{chen2019progressive} and Structured Prediction (SP) \cite{wang2017zero}. NCP calculates the $\ell_2$-norm distance between $\mathbf{x}_{t}$ and class centers clustered from source domain features and labels, thereby computing its conditional probability of belonging to the $j$-th category:
\begin{equation}
	\label{NCP}
	{p_{1}(j|\mathbf{x}_t)=\frac{\exp\left(-\left\|\mathbf{P}^\top\mathbf{x}_t-\mathbf{P}^\top\mathbf{\hat{x}}_{s_j}\right\|_2^2\right)}{\sum\limits_{j=1}^c\exp\left(-\left\|\mathbf{P}^\top\mathbf{x}_t-\mathbf{P}^\top\mathbf{\hat{x}}_{s_j}\right\|_2^2\right)}}
\end{equation}
where ${\mathbf{\hat{x}}_{s_j}}$ denotes the source center of the ${j}$-th class. By contrast, SP pseudo-labels by measuring target samples that are close to the corresponding cluster centers of target domain. The conditional probability can be calculated by:
\begin{equation}
	\label{SP}
	{p_{2}(j|\mathbf{x}_t)=\frac{\exp\left(-\left\|\mathbf{P}^\top\mathbf{x}_t-\mathbf{P}^\top\mathbf{\hat{x}}_{t_j}\right\|_2^2\right)}{\sum\limits_{j=1}^c\exp\left(-\left\|\mathbf{P}^\top\mathbf{x}_t-\mathbf{P}^\top\mathbf{\hat{x}}_{t_j}\right\|_2^2\right)}}
\end{equation}
where $\mathbf{\hat{x}}_{t_j}$ is the target cluster center of ${j}$-th class. According to Selective Pseudo-Labeling (SPL) \cite{wang2020unsupervised}, pseudo-label $\mathbf{\hat{y}}_{t_i}$ for each $\mathbf{x}_{t_i}$ is obtained by leveraging the complementarity of above two mechanisms:
\begin{equation}
	\label{SPL1}
	\pi_{ij} = \max\{p_1(j|\mathbf{x}_{t_i}), p_2(j|\mathbf{x}_{t_i})\}, \quad \forall j \in \{1,\ldots,c\}
\end{equation}
\begin{equation}
	\label{SPL3}
	\mathbf{\hat{y}}_{t_i} = \arg\max_{j}\{{\pi}_{ij}\}
\end{equation}
To facilitate the semantic consistency alignment, we sort the probability sequence as $\pi_{\text{sort}}^{(i,1)}\geq \pi_{\text{sort}}^{(i,2)} \geq \cdots \geq \pi_{\text{sort}}^{(i,c)}$ for subsequent confidence margin computation.
\section{Optimization}
\subsection{Solution Process}
In this subsection, we give detailed solving process of proposed PSCA. The overall objective optimization functions of both stages are given as follows:
\begin{equation}
	\begin{aligned}
		\label{Stage_One1}
		&\min_{\mathbf{P},\mathbf{O}^{\top}\mathbf{O}=\mathbf{I}_{c}}\sum_{i=1}^{n}\sum_{j=1}^{c}{\widetilde{y}}_{ij}\|\mathbf{P}^{\top}\mathbf{x}_{i}-\mathbf{o}_{j}\|_{2}^{2}\\&
		+{\lambda_1\mathrm{Tr}({\mathbf{P}^\top}\mathbf{X}\mathbf{H}{\mathbf{X}^\top}\mathbf{P})}+\lambda_2\|\mathbf{P}\|_{2,1}
	\end{aligned}
\end{equation}
and
\begin{equation}
	\begin{aligned}
		\label{Stage_Two2}
		&\begin{aligned}
			\min_{\mathbf{W}_s,\mathbf{W}_t,\mathbf{B}_s,\mathbf{B}_t}&\|\mathbf{W}_s\mathbf{D}_{s}-\mathbf{B}_{s}\|_{F}^{2}+\|\mathbf{W}_t\mathbf{D}_{t}-\mathbf{B}_{t}\|_{F}^{2}\\&
			+\lambda_3\|\mathbf{W}_s-\mathbf{W}_{t}\|_{F}^{2}
		\end{aligned}\\
		&\text{s.t.}~\mathbf{W}_s{\mathbf{W}_s}^{\top}=\mathbf{I}_{{r}},\mathbf{W}_t{\mathbf{W}_t}^{\top}=\mathbf{I}_{{r}},\mathbf{B}_s,\mathbf{B}_t \in \{-1,1 \}^{r\times *},
	\end{aligned}
\end{equation}
where $*\in \{n_s,n_t \} $. Eqs. \eqref{Stage_One1} and \eqref{Stage_Two2} contain multiple variables which are nontrivial to be solved simultaneously, a feasible solution is developing an alternating strategy to transform them into a series of convex subproblems. The detailed updating schemes of optimization are presented as:

\textit{\textbf{Optimization Procedure of Solving Stage One:}}

\textbf{{Update $\mathbf{P}$:}} In general, the dimension of common projected subspace is much smaller than the original data's, i.e, $q<<d$. Thus, the $\ell_{2,1}$-norm $\|\mathbf{P}\|_{2,1}$ is a regularization term with row-sparsity constraint, which aims to filter out certain redundant information in $\mathbf{X}$ by forcing most of unimportant rows in $\mathbf{P}$ shrink to zeros. It is defined as:
\begin{equation}
	\begin{aligned}
		\label{21_norm}
		\|\mathbf{P}\|_{2,1} = \sum_{i=1}^{d}\sqrt{\sum_{j=1}^{q}({p}_{ij})^2} = \sum_{i=1}^{d}\|\mathbf{p}_i\|_2
	\end{aligned}
\end{equation}
where $\mathbf{p}_i$ is the $i$-th row of $\mathbf{P}$. Since $\|\mathbf{P}\|_{2,1}$ is a non-smooth function, thus we replace it into a trace form for subsequently deriving its closed solution \cite{nie2010efficient}:
\begin{equation}
	\begin{aligned}
		\label{21_norm_tr}
		\|\mathbf{P}\|_{2,1} {~\Leftrightarrow~ }{\mathrm{Tr}({\mathbf{P}^\top}\mathbf{A}\mathbf{P})}
	\end{aligned}
\end{equation}
where $\mathbf{A}$ denotes a diagonal sub-gradient matrix, and its elements are computed by:
\begin{equation}
	\begin{aligned}
		\label{21_norm_d}
		a_{ii} = \frac{1}{2\|\mathbf{p}_i\|^2+\text{eps}}
	\end{aligned}
\end{equation}
where $\text{eps}$ is a small constant that does not affect the result. For the first term in Eq. \eqref{Stage_One}, optimization problem can be written into trace form as below:
\begin{equation}
	\begin{aligned}
		\label{Tr}
		&\sum_{i=1}^{n}\sum_{j=1}^{c}{\widetilde{y}}_{ij}\|\mathbf{P}^{\top}\mathbf{x}_{i}-\mathbf{o}_{j}\|_{2}^{2}
		\\&=\mathrm{Tr}\big(\mathbf{P^\top{X}\mathbf{S}_{1}{X}^\top{P}}+\mathbf{O}^\top\mathbf{S}_{2}\mathbf{O}-2{\mathbf{P}^\top}\mathbf{X}\mathbf{{\widetilde{Y}}}\mathbf{O}^\top\big)
	\end{aligned}
\end{equation}
where $\mathbf{S}_{1}\in\mathbb{R}^{n\times n}$ and $\mathbf{S}_{2}\in\mathbb{R}^{c\times c}$ represent two degree matrices with diagonal elements calculated as ${\mathbf{s}_{1}}_{ii}=\sum_{i=1}^{n}\widetilde{y}_{ij}$ and ${\mathbf{s}_{2}}_{ii}=\sum_{j=1}^{c}\widetilde{y}_{ij}$, respectively. 

Combining above derivations with removing irrelevant terms, we reach the subproblem in trace form $w.r.t$ $\mathbf{P}$ as:
\begin{equation}
	\begin{aligned}
		\label{P_step1}
		&\min_{\mathbf{P}}\mathrm{Tr}\big(\mathbf{P^\top{X}\mathbf{S}_{1}{X}^\top{P}}-2{\mathbf{P}^\top}\mathbf{X}\mathbf{{\widetilde{Y}}}\mathbf{O}^\top
		\\&+\lambda_1{\mathbf{P}^\top}\mathbf{X}\mathbf{H}{\mathbf{X}^\top}\mathbf{P}+{\lambda_2{\mathbf{P}^\top}\mathbf{A}\mathbf{P}}\big)
	\end{aligned}
\end{equation}
By setting the derivative of Eq. \eqref{P_step1} $w.r.t$ $\mathbf{P}$ equals to zero, we have corresponding optimal solution of $\mathbf{P}$:
\begin{equation}
	\begin{aligned}
		\label{P_step2}
		{\mathbf{P}}=(\lambda_1\mathbf{X}\mathbf{H}{\mathbf{X}^\top}+\mathbf{X}\mathbf{S}_{1}\mathbf{X}^\top+\lambda_2\mathbf{A})^{-1}\mathbf{X}\mathbf{{\widetilde{Y}}}\mathbf{O}^\top
	\end{aligned}
\end{equation}

\textbf{{Update $\mathbf{{R}}$:}} $\mathbf{{R}}$ is introduced to measure geometric consistency and the relevant objective function is given as follows: 
\begin{equation}
	\begin{aligned}
		\label{R_1}
		\min_{\mathbf{R} \geq 0, \mathbf{R}\mathbf{1}_c = \mathbf{1}_{n_t}}\sum_{i=1}^{n_t}\sum_{j=1}^{c}\{{{r}}_{ij}^{\sigma}d_{ij}-{\psi}_{ij}\text{log}({r}_{ij})\}
	\end{aligned}
\end{equation}
where $d_{ij}$ denotes the calculated value of $\ell_{2}$-distance that $d_{ij}=\|\mathbf{P}^{\top}\mathbf{x}_{t_i}-\mathbf{o}_{j}\|^2_{2}$. Let $L(\mathbf{R})$ be the partial derivative of Eq. \eqref{R_1} $w.r.t$ $\mathbf{R}$. By setting it to 0, we have:
\begin{equation}
	L(\mathbf{R})= \sigma r_{ij}^{(\sigma-1)}d_{ij}- \frac{{\psi}_{ij}}{r_{ij}\text{ln}(2)}
\end{equation}
Then the first step of optimizing $\mathbf{R}$ is pushed forward through the gradient descent:
\begin{equation}
	\begin{aligned}
		r_{ij}^{*} = r_{ij}^{(t)} - \eta {L}(\mathbf{R})
	\end{aligned}
\end{equation}
Here, $\eta$ is a learning rate and $r_{ij}^{*}$ serves as an intermediate representation. Considering the constraint conditions $\mathbf{R} \geq 0, \mathbf{R}\mathbf{1}_c = \mathbf{1}_{n_t}$, the updating of $\mathbf{R}$ can be futher processed by defining each row of the desired $\mathbf{R}^{(t+1)}$ as $\mathbf{r}_i^{(t+1)}$, then the approximate solution can be derived through:
\begin{equation}
	\begin{aligned}\label{R_2}
		\min_{\mathbf{r}_i} \frac{1}{2}\|\mathbf{r}_i^{(t+1)}  - \mathbf{r}^{*}_i\|_2^2, \quad \text{s.t.}~\mathbf{r}_i \geq 0, \mathbf{r}_i \mathbf{1}_c = \mathbf{1}_{n_t}.
	\end{aligned}
\end{equation}
Eq. \eqref{R_2} can be solved via off-the-shelf QP tools or exploiting the strategy mentioned in \cite{feng2020regularized}.

\textbf{{Update $\mathbf{{O}}$:}} According to the trace form $w.r.t$ $\mathbf{{O}}$ shown in Eq. \eqref{Tr}. Subject to the orthogonal constraint, the subproblem of  $\mathbf{O}$ can be simplified with fixing other variables as:
\begin{equation}
	\begin{aligned}
		\label{O1}
		\min_{\mathbf{O}^{\top}\mathbf{O}=\mathbf{I}_{c}} \mathrm{Tr}\big(\mathbf{O}^{\top}\mathbf{S}_{2}\mathbf{O}-2{\mathbf{P}^\top}\mathbf{X}\mathbf{{\widetilde{Y}}}\mathbf{O}^\top\big)
	\end{aligned}
\end{equation}
To facilitate optimization, a relaxed version is utilized to solve Eq. \eqref{O1}. Concretely, we firstly set the partial derivative $w.r.t$ $\mathbf{O}$ equals to 0, and then adopt the Singular Value Decomposition (SVD) technique \cite{wang2019multi}:
\begin{equation}
	\mathbf{\hat{M}\Gamma_1{M}^\top}=\mathbf{P^\top}\mathbf{X}\mathbf{\widetilde{Y}}(\mathbf{S}_{2})^{-1}
\end{equation}
where $\mathbf{\Gamma_1} = \text{diag}(\gamma_1, \gamma_2,...,\gamma_z)$ indicates the singular value matrix, $z \leq \min(q,c)$, $\mathbf{\hat{M}}\in\mathbb{R}^{q\times k}$ and $\mathbf{M}\in\mathbb{R}^{k\times c}$ serve as the transformation matrices. Evidently, the optimal solution of $\mathbf{O}$ is obtained through:
\begin{equation}\label{O_Step}
	{\mathbf{O}}=\mathbf{\hat{M}{M}^\top}
\end{equation}

\textit{\textbf{Optimization Procedure of Solving Stage Two:}}

\textbf{{Update $\mathbf{B}_{s}$:}}
The discrete constraints $\mathbf{B}_s,\mathbf{B}_t \in \{-1,1 \}$ imposed on Eq. \eqref{Stage_Two} lead to the problem NP-hard. Previous baselines \cite{zhang2023semantic,zhang2024dynamic} optimize a relaxed version with discarding these constraints, then using a $sign$ function discretize the outputs. This may result in enlarging quantization errors and bias of generated binary codes \cite{liu2024fast}. In this study, we adopt optimization technique with Discrete Cyclic Coordinate (DCC) \cite{shen2015supervised} method to enhance the training efficiency. Let the $i$-th row of $\mathbf{B}_{s}$, $\mathbf{W}_{s}$ and ${\mathbf{D}_{s}}^{\top}$ as $\mathbf{q}_{1}$, $\mathbf{w}_{1}$ and $\mathbf{d}_{1}$. Then we optimize each of the $r$ bits and derive the closed-form solution of $\mathbf{B}_{s}$:
\begin{equation}\label{bs}
	\mathbf{q}_{1}= \text{sgn}(\mathbf{w}_{1}\mathbf{d}_{1})
\end{equation}

\textbf{Update $\mathbf{B}_{t}$:} 
Similar to update $\mathbf{B}_{s}$, we optimize $\mathbf{B}_{t}$ by the following objective function:
\begin{equation}\label{bt}
	\mathbf{q}_{2}= \text{sgn}(\mathbf{w}_{2}\mathbf{d}_{2})
\end{equation}
where $\mathbf{q}_{2}$, $\mathbf{w}_{2}$ and $\mathbf{d}_{2}$ denote the $i$-th row of $\mathbf{B}_{t}$, $\mathbf{W}_{t}$ and ${\mathbf{D}_{t}}^{\top}$ respectively.

\textbf{Update $\mathbf{W}_{s}$:} By fixing the irrelevant variables, Eq. \eqref{Stage_Two} can be simplified as:
\begin{equation}
	\begin{aligned}
		\label{Ws}
		\min_{\mathbf{W}_s{\mathbf{W}_s}^{\top}=\mathbf{I}_{{r}}}\|\mathbf{W}_s\mathbf{D}_{s}-\mathbf{B}_{s}\|_{F}^{2}+\lambda_3\|\mathbf{W}_s-\mathbf{W}_{t}\|_{F}^{2}
	\end{aligned}
\end{equation}
Let $L({\mathbf{W}_s})$ be the partial derivative of Eq. \eqref{Ws} $w.r.t$ ${\mathbf{W}_s}$, $L({\mathbf{W}_s})$ is represented as below:
\begin{equation}
	\begin{aligned}
		\label{Ws1}
		L({\mathbf{W}_s})= ({\mathbf{B}_s}{\mathbf{D}_s}^{\top}+\lambda_3{\mathbf{W}_t})({\mathbf{D}_s}{\mathbf{D}_s}^{\top}+\lambda_3\mathbf{I}_{\mathcal{C}})^{-1}
	\end{aligned}
\end{equation}
For imposed orthogonal constraint, SVD technique is utilized to derive an approximate solution of ${\mathbf{W}_s}^{(t+1)}$:
\begin{equation}
	\begin{aligned}
		\label{Ws2}
		L({\mathbf{W}_s})= \mathbf{\hat{S}\Gamma_2{S}^\top}
	\end{aligned}
\end{equation}
\begin{equation}
	\begin{aligned}
		\label{Ws3}
		{\mathbf{W}_s}^{(t+1)}= \mathbf{\hat{S}{S}^\top}
	\end{aligned}
\end{equation}
\textbf{Update $\mathbf{W}_{t}$:} By removing all variables but $\mathbf{W}_{t}$, the optimization objective of $\mathbf{W}_{t}$ can be converted into:
\begin{equation}
	\begin{aligned}
		\label{Wt}
		\min_{{\mathbf{W}_t}{\mathbf{W}_t}^{\top}=\mathbf{I}_{r}}\|\mathbf{W}_s\mathbf{D}_{s}-\mathbf{B}_{s}\|_{F}^{2}+\lambda_3\|\mathbf{W}_s-\mathbf{W}_{t}\|_{F}^{2}
	\end{aligned}
\end{equation}
Similar to updating ${\mathbf{W}_s}$, by setting the derivative of ${\mathbf{W}_t}$ as $L({\mathbf{W}_t})$, we reach the equation below:
\begin{equation}
	\begin{aligned}
		\label{Wt1}
		L({\mathbf{W}_t})= ({\mathbf{B}_t}{\mathbf{D}_t}^{\top}+\lambda_3{\mathbf{W}_s})({\mathbf{D}_t}{\mathbf{D}_t}^{\top}+\lambda_3\mathbf{I}_{\mathcal{C}})^{-1}.
	\end{aligned}
\end{equation}
Then the SVD-based updating rule of ${\mathbf{W}_t}^{(t+1)}$ is as follows:
\begin{equation}
	\begin{aligned}
		\label{Wt2}
		L({\mathbf{W}_t})= \mathbf{\hat{K}\Gamma_3{K}^\top}
	\end{aligned}
\end{equation}
\begin{equation}
	\begin{aligned}
		\label{Wt3}
		{\mathbf{W}_t}^{(t+1)}= \mathbf{\hat{K}{K}^\top}
	\end{aligned}
\end{equation}

Algorithm \ref{Alg1} presents a summary detailed optimization procedure of two stage.
\begin{algorithm}[t]
	\caption{Training Process of PSCA}\label{Alg1}
	\small
	\KwIn{Both domain samples ${\mathbf{X}_s} \in {\mathbb{R}^{d \times n_s}}$ and ${\mathbf{X}_t} \in {\mathbb{R}^{d \times n_t}}$; training labels ${\mathbf{Y}_s\in\mathbb{R}^{n_s\times c}}$; hash code length $r$; maximum iteration $T$; subspace dimension $q$; parameters $\lambda_1$, $\lambda_2$, $\lambda_3$, $\sigma$.}
	\KwOut{The optimal variable collection $\mathcal{W} = \{\mathbf{P}$, $\mathbf{R}$, $\mathbf{O}$, $\mathbf{B}_s$,$ \mathbf{B}_t$, $\mathbf{W}_s$, $\mathbf{W}_t\}$.}
	\tcp{Stage One: Prototype-based Semantic Consistency Alignment}
	Generate pseudo-labels ${\mathbf{\hat{Y}}_{t}\in\mathbb{R}^{n_t\times c}}$; initialize $\mathbf{{R}}=\mathbf{\hat{Y}}_{t}$; concatenate the whole semantic matrix as $\mathbf{\widetilde{Y}}=[\mathbf{{Y}}_{s},\mathbf{R}]^\top$.  
	
	Initialize projection matrix $\mathbf{P} \in {\mathbb{R}^{d \times q}}$. 
	
	Initialize class prototypes $\mathbf{O}\in\mathbb{R}^{q\times c}$ by a centroid calculation function $\mathcal{K}=\frac{\sum_{i=1}^{n} {\mathbf{p}}^{\top}{\mathbf{x}_{i}}{\widetilde{y}_{ij}}}{\sum_{i=1}^{n} {\widetilde{y}_{ij}}}$, $\forall 1 \leq j \leq c$. \\
	\While{$t \leqslant T$}{
		Update $\mathbf{P}$ by using Eq. \eqref{P_step2};\\
		Update $\mathbf{R}$ by using Eq. \eqref{R_2};\\
		Update $\mathbf{O}$ by using Eq. \eqref{O_Step};
	}
	\tcp{Feature Reconstruction}
	Reconstruct source feature representation $\mathbf{\widetilde{X}}_{s} = [\mathbf{\widetilde{x}}_{s_1}, \mathbf{\widetilde{x}}_{s_2},\dots,\mathbf{\widetilde{x}}_{s_{ns}}]$, where $\mathbf{\widetilde{x}}_{s_i} = \sum_{m}^{c}{y}_{im}\mathbf{o}_{m}^\top$.
	
	Reconstruct source feature representation $\mathbf{\widetilde{X}}_{t} = [\mathbf{\widetilde{x}}_{t_1}, \mathbf{\widetilde{x}}_{t_2},\dots,\mathbf{\widetilde{x}}_{t_{nt}}]$, where $\mathbf{\widetilde{x}}_{t_i} = \sum_{m}^{c}{r}_{im}\mathbf{o}_{m}^\top$.
	
	Fuse  $\mathbf{\widetilde{X}}=[\mathbf{\widetilde{X}}_{s},\mathbf{\widetilde{X}}_{t}]$ with the projected feature $\mathbf{P}^{\top}\mathbf{X}$ as $\mathbf{D}=[{\mathbf{D}_{s}},{\mathbf{D}_{t}}]=[{\mathbf{d}_{s}}_1,...,{\mathbf{d}_{s}}_{ns},{\mathbf{d}_{t}}_1,...,{\mathbf{d}_{t}}_{nt}]=[\mathbf{\widetilde{X}},\mathbf{X}^\top\mathbf{P}]^\top\in\mathbb{R}^{\mathcal{C}\times n}$, where $\mathcal{C}=q+d$.\\
	\tcp{Stage Two: Hash Learning}
	Initialize orthogonally $\mathbf{W_s}$, $\mathbf{W_t}$.\\
	
	Initialize $\mathbf{B}_s$, and $\mathbf{B}_t$ randomly.\\
	\While{$t \leqslant T$}{
		Update $\mathbf{B}_s$ by using Eq. \eqref{bs};\\
		Update $\mathbf{B}_t$ by using Eq. \eqref{bt};\\
		Update $\mathbf{W}_s$ by using Eq. \eqref{Ws3};\\
		Update $\mathbf{W}_t$ by using Eq. \eqref{Wt3};
	}
\end{algorithm}
\begin{table*}[t]
	\centering
	\fontsize{9}{12}\selectfont
	\setlength{\tabcolsep}{1.4mm}
	{
		\begin{tabular}{c|cccc|c}
			\hline
			Baseline  & PWCF & DAPH* & TSS & DCS-LSG & PSCA  \\ \hline
			MNIST$\to$USPS  & 181.0 & 2.8 & 28.3 & 2.9 & 1.8\\
			COIL1$\to$COIL2   & 518.2 & 15.1 & 16.4 & 5.3 & 4.6 \\
			A$\to$D      &10621.5 & 356.9& 205.9 & 193.6 & 212.3 \\
			P$\to$R     &38011.3 & 607.5& 408.9 &252.2 & 240.2  \\\hline
			Computation Complexity &$\mathcal{O}(n^3+d^2n)$ &$\mathcal{O}(n^2d+d^2n+d^3)$ & $\mathcal{O}(n^2d+d^3)$& $\mathcal{O}(n^2d+d^2n+d^3)$ &$\mathcal{O}(n^2d+d^3+{n_s}^2d+{n_t}^2d)$\\
			\hline
		\end{tabular} 
	}
	\caption{Running time (s) comparison of PSCA and multiple baselines across different benchmark datasets at 64 Bits.}
	\label{table_time}
\end{table*}
\subsection{Algorithm Analysis}
As a domain adaptive retrieval hashing algorithm, PSCA is expected to have good convergence property and efficient retrieval speed. Thus, beginning with analysing on convergence of PSCA, this subsection further studies computational complexity analysis and running time comparisons.

\textit{\textbf{Convergence:}} Notice that the objective functions, i.e., Eqs. \eqref{Stage_One1} and  \eqref{Stage_Two2} of PSCA, are nonconvex to be solved. We adopt a surrogate strategy for optimizing variable collection $\mathcal{W} = \{\mathbf{P}$, $\mathbf{R}$, $\mathbf{O}$, $\mathbf{B}_s$, $\mathbf{B}_t$, $\mathbf{W}_s$,$ \mathbf{W}_t\}$ as Definition 1.
\begin{definition}
	For the nonconvex optimization problem, its approximate closed-form solution can be derived by transforming it into solving several convex subproblems.
\end{definition}
Define $\mathcal{L}(\mathcal{W})$ be the overall objective function value of two stages. By definition 1, the proposed Algorithm \ref{Alg1} ensures the monotonic descent of $\mathcal{L}(\mathcal{W})$ during iterative process, conditional on all subproblems reaching their global optima. For theoretical analysis, we verify the optimization procedures as Lemma 1.
\begin{lemma}
	Algorithm \ref{Alg1} guarantees that $\mathcal{L}(\mathcal{W})$ value decreases monotonically and eventually reaches convergence.
\end{lemma}
\begin{proof}
	Notice that Eqs. \eqref{Stage_One1} and \eqref{Stage_Two2} are consequentially bounded below since that the strictly positive nature of the Frobenius norm sum (at least bounded below by a constant $\omega \geq 0$). Following the detailed derivations in subsection \textit{Solution Process}, all variables can reach their optimal closed-form solutions. In other words, we have following monotonic sequence during each iteration:
	\begin{equation}
		\begin{aligned}
			&\mathcal{L}(\mathcal{W}^{T}) \geq \mathcal{L}(\mathbf{P}^{T-1}, \mathbf{R}^T, \mathbf{O}^T, {\mathbf{B}_s}^T, {\mathbf{B}_t}^T, {\mathbf{W}_s}^T, {\mathbf{W}_t}^T) 
			\\&\geq \mathcal{L}(\mathbf{P}^{T-1}, \mathbf{R}^{T-1}, \mathbf{O}^T, {\mathbf{B}_s}^T, {\mathbf{B}_t}^T, {\mathbf{W}_s}^T, {\mathbf{W}_t}^T) 
			\\&\geq \mathcal{L}(\mathbf{P}^{T-1}, \mathbf{R}^{T-1}, \mathbf{O}^{T-1}, {\mathbf{B}_s}^T, {\mathbf{B}_t}^T, {\mathbf{W}_s}^T, {\mathbf{W}_t}^T) 
			\\&\geq \mathcal{L}(\mathbf{P}^{T-1}, \mathbf{R}^{T-1}, \mathbf{O}^{T-1}, {\mathbf{B}_s}^{T-1}, {\mathbf{B}_t}^T, {\mathbf{W}_s}^T, {\mathbf{W}_t}^T) 
			\\&\geq \mathcal{L}(\mathbf{P}^{T-1}, \mathbf{R}^{T-1}, \mathbf{O}^{T-1}, {\mathbf{B}_s}^{T-1}, {\mathbf{B}_t}^{T-1}, {\mathbf{W}_s}^T, {\mathbf{W}_t}^T)
			\\&\geq \mathcal{L}(\mathbf{P}^{T-1}, \mathbf{R}^{T-1}, \mathbf{O}^{T-1}, {\mathbf{B}_s}^{T-1}, {\mathbf{B}_t}^{T-1}, {\mathbf{W}_s}^{T-1}, {\mathbf{W}_t}^T)
			\\&	\geq\mathcal{L}(\mathcal{W}^{T-1})
		\end{aligned}
	\end{equation}
	The iterative sequence confirms that progressively updating each variable guarantees monotonic reduction of the objective function value $\mathcal{L}(\mathcal{W})$. On the basis of the Bounded Monotone Convergence theory \cite{russell2020principles}, Algorithm \ref{Alg1} achieves good convergence property of PSCA.\renewcommand{\qedsymbol}{$\blacksquare$}
\end{proof}

For visual analysis, we plot convergence curves on all benchmark datasets. As illustrated in Figure \ref{COV}, the objective function value exhibits fast convergence, reaching stable values within 15 iterations. The decreasing loss confirms algorithmic good convergence.

\begin{figure}[t]
	\centering
	\begin{center}
		\centering
		\subfigure[MNIST$\to$USPS]{
			\includegraphics[width=0.22\textwidth]{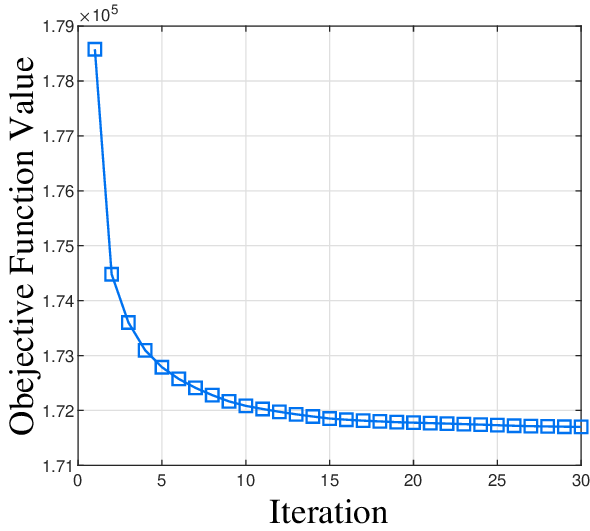}}
		\subfigure[COIL1$\to$COIL2]{
			\includegraphics[width=0.22\textwidth]{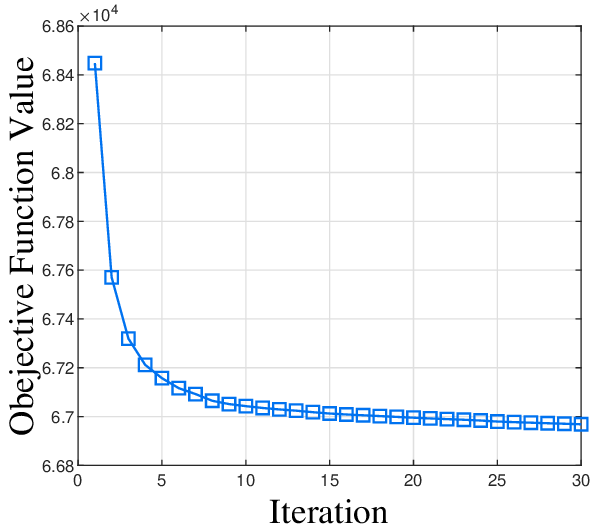}}
		\subfigure[A$\to$D]{
			\includegraphics[width=0.22\textwidth]{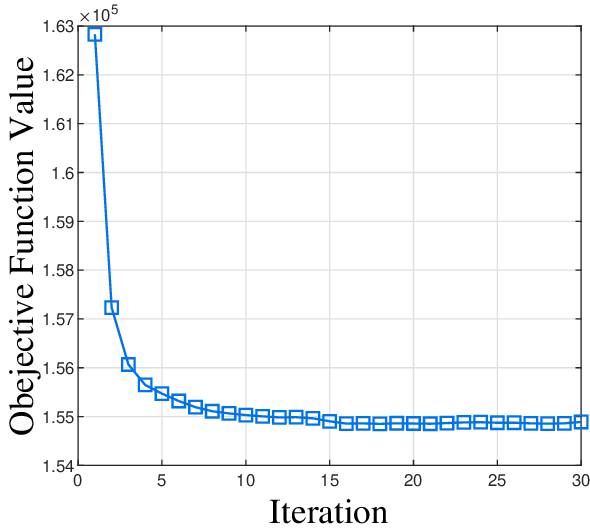}}
		\subfigure[P$\to$R]{
			\includegraphics[width=0.22\textwidth]{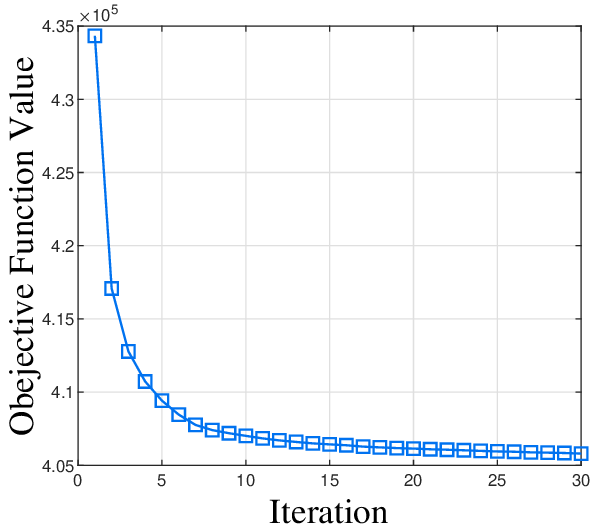}}
		\caption{Convergence analysis of PSCA.}
		\label{COV}
	\end{center}
\end{figure}

\textit{\textbf{Computational Complexity:}} Note that feature reconstruction involves simple linear combination operations, thus the computational burden of the proposed PSCA mainly comes from the optimization procedures of two stages. We report the computational complexity analyses of each optimization step as follows: $\mathbf{P}_{step} = \mathcal{O}(d^2n+dn^{2}+d^{3}+dnc+d^{2}q+dcq)$; $\mathbf{R}_{Step} = \mathcal{O}(dq{n_t}+qc{n_t}+{n_t}{c}^2)$; $\mathbf{O}_{Step} = \mathcal{O}(qdn+qnc+q^{2}c+qc^{2})$; ${\mathbf{B}_s}_{Step} = \mathcal{O}(r\mathcal{C}n_s)$; ${\mathbf{B}_t}_{Step} = \mathcal{O}(r\mathcal{C}n_t)$; ${\mathbf{W}_s}_{Step} = \mathcal{O}(rn_s\mathcal{C}+\mathcal{C}{n_s}^2+\mathcal{C}^3+r\mathcal{C}^2)$; ${\mathbf{W}_t}_{Step} = \mathcal{O}(rn_t\mathcal{C}+\mathcal{C}{n_t}^2+\mathcal{C}^3+r\mathcal{C}^2)$. The time complexity of adaptively weighting $\alpha$ is $\mathcal{O}({n_t}cq)$.

Seeing that $r,c<q<<d<n$ where $n=n_s+n_t$, thus for $T$ iterations, the overall computation complexity of PSCA can be approximated by $\mathcal{O}(T(n^2d+d^3+{n_s}^2d+{n_t}^2d)$. Benefitting from our efficient optimization strategies, PSCA exhibits competitive training efficiency and retrieval speed, we now conduct the running time analysis. 

\textit{\textbf{Running Time Comparison:}} Table \ref{table_time} reports the training time comparison of PSCA with baselines PWCF \cite{huang2020probability}, DAPH* \cite{huang2021domain}, TSS \cite{chen2023two}, DCS-LSG \cite{zhang2024dynamic}. Experiments are conducted uniformly on a computer with an Intel(R) Core(TM) i9-9900K CPU @ 3.60GHz, 32GB RAM and a 64-bit Windows operating system. It can be observed that the training time depends on both dataset size and feature dimensions. We can make the following observations:
\begin{itemize}
	\item {\texttt{}} Regarding the impact of dataset size, we analyze the experimental outcomes from two cases: A$\to$D (4,096$\times$3,315) and P$\to$R (4,096$\times$8,796). Our findings reveal that as the sample quantity expands, the training time requirements exhibit a proportional increase. This observation indicates that algorithm's computational overhead scales linearly with dataset volume.
	\item {\texttt{}} Concerning feature dimensionality effects, we conduct comparative experiments using MNIST$\to$USPS (256$\times$3,800) and A$\to$D (4,096$\times$3,315) datasets. While these datasets contain comparable sample sizes, they differ substantially in their feature space dimensions. The experimental outcomes demonstrate that when feature dimensionality increases, the algorithm's training duration exhibits significant growth patterns.
\end{itemize}
Benefiting from orthogonal class prototype learning, PSCA avoids inefficient pair-wise sample alignment strategies, and this approach contributes to its computational efficiency. As demonstrated by the running time comparison on cases A$\to$D and P$\to$R, PSCA exhibits greater robustness to dimensionality growth compared to baseline methods. Overall, we conclude that the computational efficiency of PSCA proves advantageous for high-dimensional data processing.

\section{Experiment}
\subsection{Parameter Sensitivity}
While our experimental results demonstrate superior performance of PSCA across various datasets, it is essential to investigate how sensitive the method is to parameter choices, which directly affects its robustness and applicability.
\begin{figure}[t]
	\centering
	\begin{center}
		\centering
		\subfigure[$\lambda_1$]{\label{PSA}
			\includegraphics[width=0.22\textwidth]{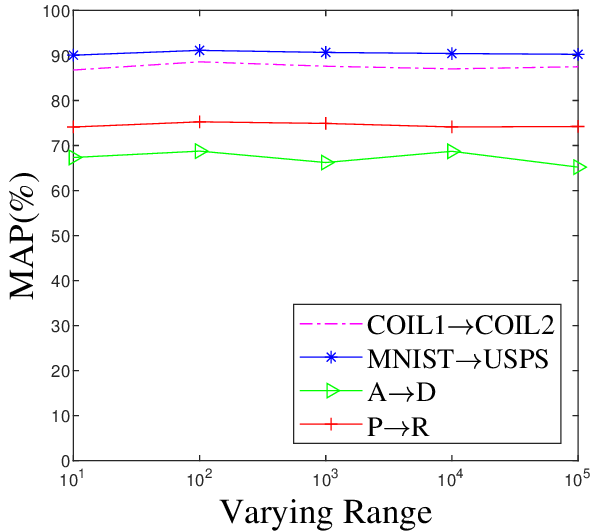}}
		\subfigure[$\lambda_2$]{\label{PSB}
			\includegraphics[width=0.22\textwidth]{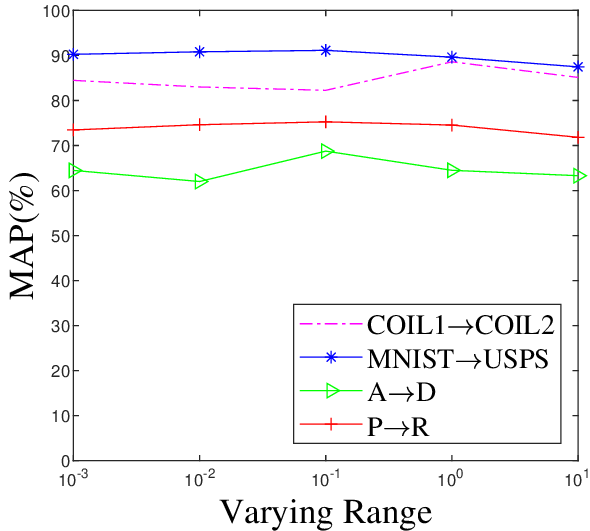}}
		\subfigure[$\lambda_3$]{\label{PSC}
			\includegraphics[width=0.22\textwidth]{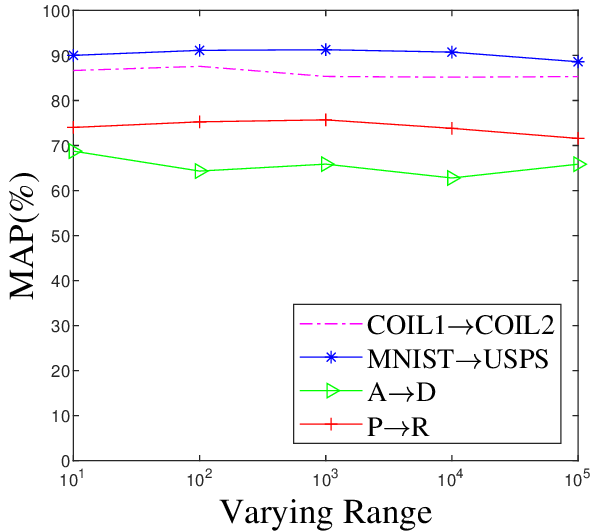}}
		\subfigure[$\alpha$]{\label{PSD}
			\includegraphics[width=0.22\textwidth]{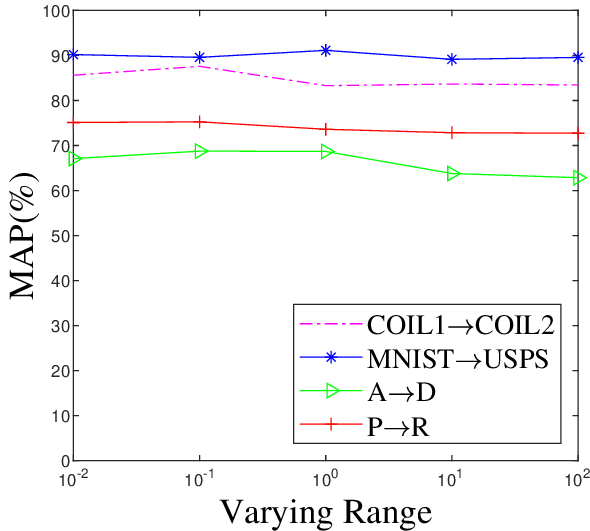}}
		\caption{Parameter sensitivity analysis of PSCA on cases MNIST$\rightarrow$USPS, COIL1$\rightarrow$COIL2, A$\rightarrow$D and P$\rightarrow$R.}
		\label{PS}
	\end{center}
\end{figure}
The proposed PSCA framework involves several key trade-off parameters that control different aspects of the learning process: 1) $\lambda_1$ controls the weight of domain marginal alignment through MMD, 2) $\lambda_2$ regulates the sparsity constraint on the projection matrix ${\mathbf{P}}$, and 3) $\lambda_3$ balances the approximation between dual domain-specific quantization functions. Additionally, to analyze the effectiveness of our semantics consistency alignment, we manually adjusts the semantic fusion strength factor $\alpha$ instead of adaptive strategy, where larger $\alpha$ values encourage $\psi$ increasing, indicating greater trust placed in the pseudo-labels.

To study the stability and robustness of PSCA, we systematically adjust each parameter while keeping others fixed within a reasonable range, then evaluate the resulting MAP scores across different parameter combinations. Figure \ref{PS} illustrates the sensitivity curves showing how MAP scores vary under different parameter combinations. Specifically, $\lambda_1$ and $\lambda_3$ vary within the range of [$10^{1}$, ${10}^{2}$, $\cdots$, ${10}^{5}$]. Varying ranges for $\lambda_2$ and $\alpha$ are as [${10}^{-3}$, ${10}^{-2}$, $\cdots$, ${10}^{1}$] and [${10}^{-2}$, ${10}^{-1}$, $\cdots$, ${10}^{2}$] respectively. By analysing Figure \ref{PS}, we draw the following two observations:
\begin{itemize}
	\item \textit{\textbf{Parameter Robustness:}} As illustrated in Figures \ref{PSA}, \ref{PSB} and \ref{PSC}, PSCA demonstrates remarkable stability across varying parameter configurations for $\lambda_1$, $\lambda_2$, and $\lambda_3$. The MAP scores remain relatively consistent within reasonable parameter ranges, indicating that PSCA is not overly sensitive to parameter selection.
	\item \textit{\textbf{Geometric Proximity Validation:}} Figure \ref{PSD} reveals a insightful pattern regarding the semantic fusion strength $\alpha$. Across all datasets, MAP scores initially improve with increasing $\alpha$ values but begin to deteriorate after reaching optimal thresholds. This degradation occurs when pseudo-label semantics start to dominate the prototype learning process, gradually weakening the contribution of geometric guidance. This phenomenon indicates that our semantic consistency alignment effectively mitigates the adverse effects of erroneous semantics, as excessive reliance on potentially noisy pseudo-labels inevitably compromises retrieval effectiveness.
\end{itemize}
In summary, we can draw the conclusion that the parameters of PSCA are not particularly sensitive within a reasonable range of variation.

\end{document}